\documentclass{article}

\usepackage{amsmath}
\usepackage[arrowdel]{physics}
\usepackage[pdftex]{graphicx}
\usepackage{caption}
\usepackage{subcaption}
\usepackage[square,numbers]{natbib}
\usepackage[preprint]{neurips_2022}

\usepackage[utf8]{inputenc} 
\usepackage[T1]{fontenc}    
\usepackage{hyperref}       
\usepackage{url}            
\usepackage{booktabs}       
\usepackage{amsfonts}       
\usepackage{nicefrac}       
\usepackage{microtype}      
\usepackage{xcolor}         

\usepackage[english]{babel}
\usepackage{amsthm}

\newtheorem{theorem}{Theorem}
\newtheorem{corollary}{Corollary}
\newtheorem{lemma}{Lemma}

\theoremstyle{definition}
\newtheorem{definition}{Definition}

\theoremstyle{property}
\newtheorem{property}{Property}

\theoremstyle{remark}
\newtheorem*{remark}{Remark}

\title{ATLAS: Universal Function Approximator for Memory Retention}

\author{%
    Heinrich P.~van Deventer\thanks{\url{https://github.com/hpdeventer}} \\
  Department of Computer Science\\
  University of Pretoria\\
  \texttt{HPDeventer@gmail.com} \\
  \And
  Anna S.~Bosman\thanks{\url{https://annabosman.github.io/}} \\
  Department of Computer Science\\
  University of Pretoria\\
  \texttt{anna.bosman@up.ac.za} \\
}

\begin{document}

\maketitle

\begin{abstract}

Artificial neural networks (ANNs), despite their universal function approximation capability and practical success, are subject to catastrophic forgetting. Catastrophic forgetting refers to the abrupt unlearning of a previous task when a new task is learned. It is an emergent phenomenon that hinders continual learning. Existing universal function approximation theorems for ANNs guarantee function approximation ability, but do not predict catastrophic forgetting. This paper presents a novel universal approximation theorem for multi-variable functions using only single-variable functions and exponential functions. Furthermore, we present \emph{ATLAS}—a novel ANN architecture based on the new theorem. It is shown that ATLAS is a universal function approximator capable of some memory retention, and continual learning. The memory of ATLAS is imperfect, with some off-target effects during continual learning, but it is well-behaved and predictable. An efficient implementation of ATLAS is provided. Experiments are conducted to evaluate both the function approximation and memory retention capabilities of ATLAS.

\end{abstract}

\section{Introduction}\label{sec:introduction}

Catastrophic forgetting~\cite{french1999catastrophic, kemker2018measuring, robins1995catastrophic} is an emergent phenomenon where a machine learning model such as an artificial neural network (ANN) learns a new task, and the subsequent parameter updates interfere with the model's performance on previously learned tasks. Catastrophic forgetting is also called catastrophic interference~\cite{mccloskey1989catastrophic}. If an ANN cannot effectively learn many tasks, it has limited utility in the context of continual learning~\cite{hadsell2020embracing, kaushik2021understanding}. Catastrophic forgetting is like learning to pick up a cup, but simultaneously forgetting how to breathe. Even linear functions are susceptible to catastrophic forgetting, as illustrated in Figure~\ref{fig:fig_linear_function}

 \begin{figure}[!h]
\centering
\includegraphics[width=0.5\linewidth]{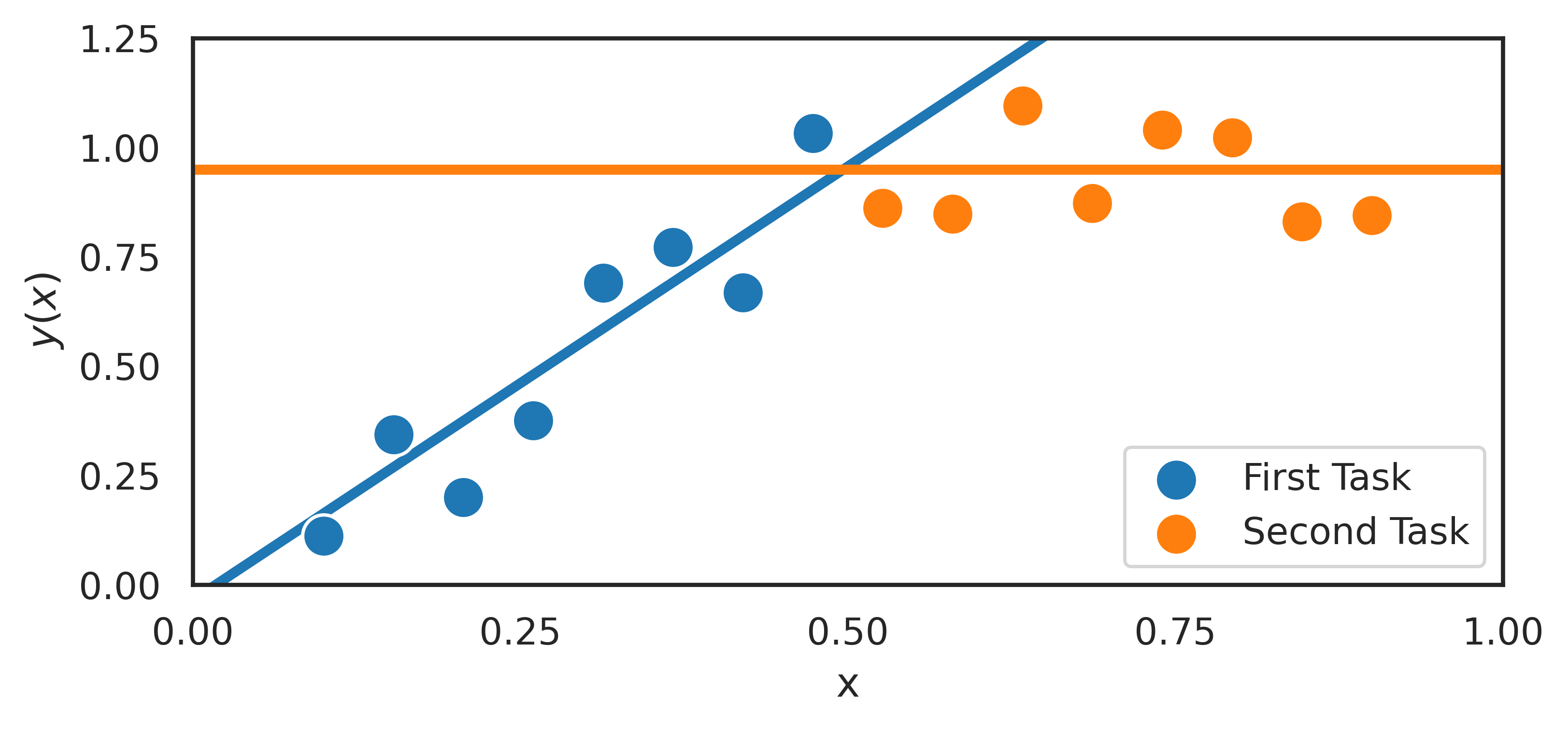}
\caption{A linear function is susceptible to catastrophic forgetting.}
\label{fig:fig_linear_function}
\end{figure}

The simple example of a linear regression model being susceptible to catastrophic forgetting might be due to the non-linearity of the target function, noise, or parameter sharing across the input. Parameter sharing is avoidable with piece-wise defined functions such as splines~\citep{kasam-paper}. ANNs can be explained in many ways; a useful analogy is to compare ANNs to very large lookup tables that store information. Removing and updating values has off-target effects throughout the table or ANN.

Universal function approximation theorems are a cornerstone of machine learning, and prove that ANNs can approximate any given continuous target function~\cite{hanin2019universal, hornik1989multilayer, kratsios2021universal} under certain assumptions. The theorems do not specify how to find an ANN with sufficient performance for problems in practice. Gradient descent optimisation is the convention for finding/training neural networks, but other optimisation and learning procedures exist~\citep{NEURIPS2020_e7a425c6}. ATLAS models trained with gradient descent methods exhibit desirable properties. However, other optimisation techniques like evolutionary algorithms may not elicit the same properties.

This paper introduces ATLAS—a novel universal function approximator based on B-splines that has some intrinsic memory retention, even in the absence of other training and regularisation techniques. ATLAS has well-behaved parameter gradients that are sparse, bounded and orthogonal between input points that are far enough from each other. The accompanying representation and universal approximation theorems are also provided.

\section{Relevant Studies}\label{sec:relevant_studies}

It is conjectured that overlapping representations in ANNs lead to catastrophic forgetting~\citep{kaushik2021understanding}. Catastrophic forgetting occurs when parameters necessary for one task change while training to meet the objectives of another task~\citep{kirkpatrick2017overcoming, mcrae1993catastrophic}. The least desirable strategy to mitigate catastrophic forgetting is retraining a model over all tasks. Regularisation techniques like elastic weight consolidation (EWC) have also been employed~\citep{kirkpatrick2017overcoming}. Data augmentation approaches such as rehearsal and pseudo-rehearsal have also been employed~\citep{robins1995catastrophic}. Other ideas from optimal control theory in combination with dynamic programming have also been applied to counteract catastrophic forgetting, with a cost functional similar in form to the action integral from physics and Lagrangian mechanics~\citep{MetaContinualLearningviaDynamicProgramming}.

Orthogonal Gradient Descent (OGD) is a training augmentation or optimisation technique that modifies the gradient updates of subsequent tasks to be orthogonal to previous tasks~\citep{orthogonalgradientdescent2021,bennani2021generalisation}. One can describe data in terms of a distribution defined over the input space, target values, and time (the order of data or tasks that are presented during training). OGD attempts to make gradient updates orthogonal to each other over time. ATLAS, in contrast, possesses distal orthogonality, meaning that if two inputs are far enough from each other in the input space, then corresponding gradient updates will be orthogonal. A corollary of this is that if the data distribution between tasks shifts in the input space, then the subsequent gradient updates will tend to be orthogonal. ATLAS does not use external memory like OGD. Extensions of OGD include PCA-OGD, which compresses gradient updates into principal components to reduce memory requirements~\citep{pmlr-v130-doan21a}. The Neural Tangent Kernel (NTK) overlap matrices, as discussed by~\citet{pmlr-v130-doan21a}, could be a useful tool for analysing ATLAS models. 

The survey by~\citet{Delange_2021} gives an extensive overview of continual learning to address catastrophic forgetting. ATLAS is a model that implements parameter isolation, because of its use of piece-wise defined splines. Particularly relevant to ATLAS is the work on scale of initialisation and extreme memorisation~\citep{ExtremeMemorizationviaScaleofInitialization}. Increasing the density of basis functions in ATLAS can lead to better memorisation, and increases the scale of some parameters in ATLAS which may affect generalisation.

Pi-sigma neural networks use nodes that compute products instead of sums~\citep{pi-sigma}. Pi-sigma neural networks have some similarities with the global structure of ATLAS. B-splines, which form the basis of ATLAS, have been applied for machine learning~\citep{douzette2017b}. \citet{scardapane2017learning} investigated trainable activation functions parameterised by splines. 
Uniform cubic B-splines have basis functions that are translates of one another~\citep{Branson04apractical}. Uniform cubic B-splines have been tested for memory retention, and ATLAS is an improvement on existing spline models~\citep{kasam-paper}.

B-splines, and by extension ATLAS, can be trained to fit lower frequency components, expanded and trained again until a network is found with sufficient accuracy and generalisation, similar to other techniques~\citep{lane1991multi,FunctionalRegularization}. It is not necessary to expand the capacity of an ATLAS model to learn new tasks, as with some other approaches~\citep{progressiveneuralnetworks2016}. ATLAS does in practice demonstrate something akin to "graceful forgetting" as discussed in~\citet{graceful-forgetting-continual-learning-via-neural-pruning}.

\section{Notation}
Vector quantities like $\vec{\mathbf{x}}$ are clearly indicated with a bar or arrow for legibility. Parameters, inputs, functions etc. without a bar or arrow are scalar quantities like $S(x)$. Some scalar quantities with indices are the scalar components of a vector like $x_{j}$ or scalar parameters in the model like $ \theta_{i}$. The gradient operator that acts on a scalar function like $ \grad_{\vec{\mathbf{\theta}}} A( \vec{\mathbf{x}} )$ yields a vector-valued function $ \grad_{\vec{\mathbf{\theta}}} A( \vec{\mathbf{x}} )$ as is typical of multi-variable calculus.

\section{Exponential Representation Theorem}

Any continuous multi-variable function on a compact space can be uniformly approximated with multi-variable polynomials by the Stone-Weierstrass Theorem. Let $\mathcal{I}$ denote an index set of tuples of natural numbers including zero such that $i_{j} \in \mathbb{N}^{0}$ for all $j \in \mathbb{N}$ with $i = (i_{1},..,i_{n}) \in \mathcal{I}$ and $a_{i} \in \mathbb{R}$. Multi-variable polynomials can be represented as:

$$
y(\vec{\mathbf{x}}) 
= y(x_{1},..,x_{n} ) 
= \sum_{i \in \mathcal{I} } a_{i}  x_{1}^{i_{1}} x_{2}^{i_{2}} ... x_{n}^{i_{n}}
= \sum_{i \in \mathcal{I} } a_{i} \Pi_{j=1}^{n} x_{j}^{i_{j}} 
$$

Each monomial term $a_{i} \Pi_{j=1}^{n} x_{j}^{i_{j}} $ is a product of single-variable functions in each variable. It is desirable to rewrite products as sums using exponentials and logarithms. 


\begin{lemma}
For any $a_{i} \in \mathbb{R}$, there exists $\gamma_{i}>0$ and $\beta_{i}>0$, such that: 
$a_{i}= \gamma_{i} - \beta_{i}$
\end{lemma}


\begin{theorem}[Exponential representation theorem]
\label{thm_exp_rep}
Any multi-variable polynomial function $y(\vec{\mathbf{x}})$ of $n$ variables over the positive orthant, can be exactly represented by continuous single-variable functions $g_{i,j}(x_{j})$ and $h_{i,j}(x_{j})$ in the form:

\begin{equation*}
\begin{split}
y(\vec{\mathbf{x}}) 
= \sum_{i \in \mathcal{I} } \exp( \Sigma_{j=1}^{n} g_{i,j}(x_{j}))    - \exp( \Sigma_{j=1}^{n} h_{i,j}(x_{j}))
\end{split}
\end{equation*}
\end{theorem}


\begin{proof}
Consider any monomial term $a_{i} \Pi_{j=1}^{n} x_{j}^{i_{j}} $ with $a_{i} \in \mathbb{R}$, then by Lemma 1 there exist strictly positive numbers $\gamma_{i}>0$ and $\beta_{i}>0$, such that: 

\begin{equation*} 
\begin{split}
a_{i} \Pi_{j=1}^{n} x_{j}^{i_{j}}  
& = \gamma_{i} \Pi_{j=1}^{n} x_{j}^{i_{j}}  - \beta_{i} \Pi_{j=1}^{n} x_{j}^{i_{j}}   \\
& = \exp(\log(\gamma_{i} \Pi_{j=1}^{n} x_{j}^{i_{j}} )) - \exp(\log(\beta_{i} \Pi_{j=1}^{n} x_{j}^{i_{j}} )) \\
& = \exp(\log(\gamma_{i}) + \Sigma_{j=1}^{n} \log(x_{j}^{i_{j}} )) 
    - \exp(\log(\beta_{i}) +\Sigma_{j=1}^{n} \log(x_{j}^{i_{j}} )) \\
\end{split}
\end{equation*}

The argument of each exponential function is a sum of single-variable functions and constants. Without loss of generality, a set of single-variable functions can be defined such that:

\begin{equation*} 
\begin{split}
a_{i} \Pi_{j=1}^{n} x_{j}^{i_{j}}  
& = \exp( \Sigma_{j=1}^{n} g_{i,j}(x_{j}))
    - \exp( \Sigma_{j=1}^{n} h_{i,j}(x_{j})) 
\end{split}
\end{equation*}

Since this holds for any $a_{i} \Pi_{j=1}^{n} x_{j}^{i_{j}} $ and all $i \in \mathcal{I}$, it follows that:

\begin{equation*} 
\begin{split}
y(\vec{\mathbf{x}}) 
= \sum_{i \in \mathcal{I} } \exp( \Sigma_{j=1}^{n} g_{i,j}(x_{j}))    - \exp( \Sigma_{j=1}^{n} h_{i,j}(x_{j}))
\end{split}
\end{equation*}
\end{proof}

This result is fundamental to the paper. Since every continuous function can be approximated with multi-variable polynomials, it follows that every continuous function can be approximated with positive and negative exponential functions. Single-variable function approximators are pivotal and must be reconsidered. Universal function approximation can also be proven with the sub-algebra formulation of the Stone-Weierstrass theorem, but it's not as delightful and simple as the first constructive proof given above.

\section{Single-Variable Function Approximation}

Splines are piece-wise defined single-variable functions over some interval. Each sub-interval of a spline is most often locally given by a low degree polynomial, even though the global structure is not a low degree polynomial. B-splines are polynomial splines that are defined in a way that resembles other basis function formulations~\citep{Branson04apractical}. Each single-variable function in ATLAS is approximated with uniform cubic B-spline basis functions, shown in Figure~\ref{fig:basis_densities_comparison}. B-splines can approximate any single-variable function, similar to using the Fourier basis. With uniform B-splines, each basis function is scaled so that the unit interval is \textbf{uniformly} partitioned, as in Figure~\ref{fig:basis_densities_comparison}.

\begin{figure} [!h]
     \centering
     \begin{subfigure}[!h]{0.32\textwidth}
         \centering
         \includegraphics[width=\textwidth]{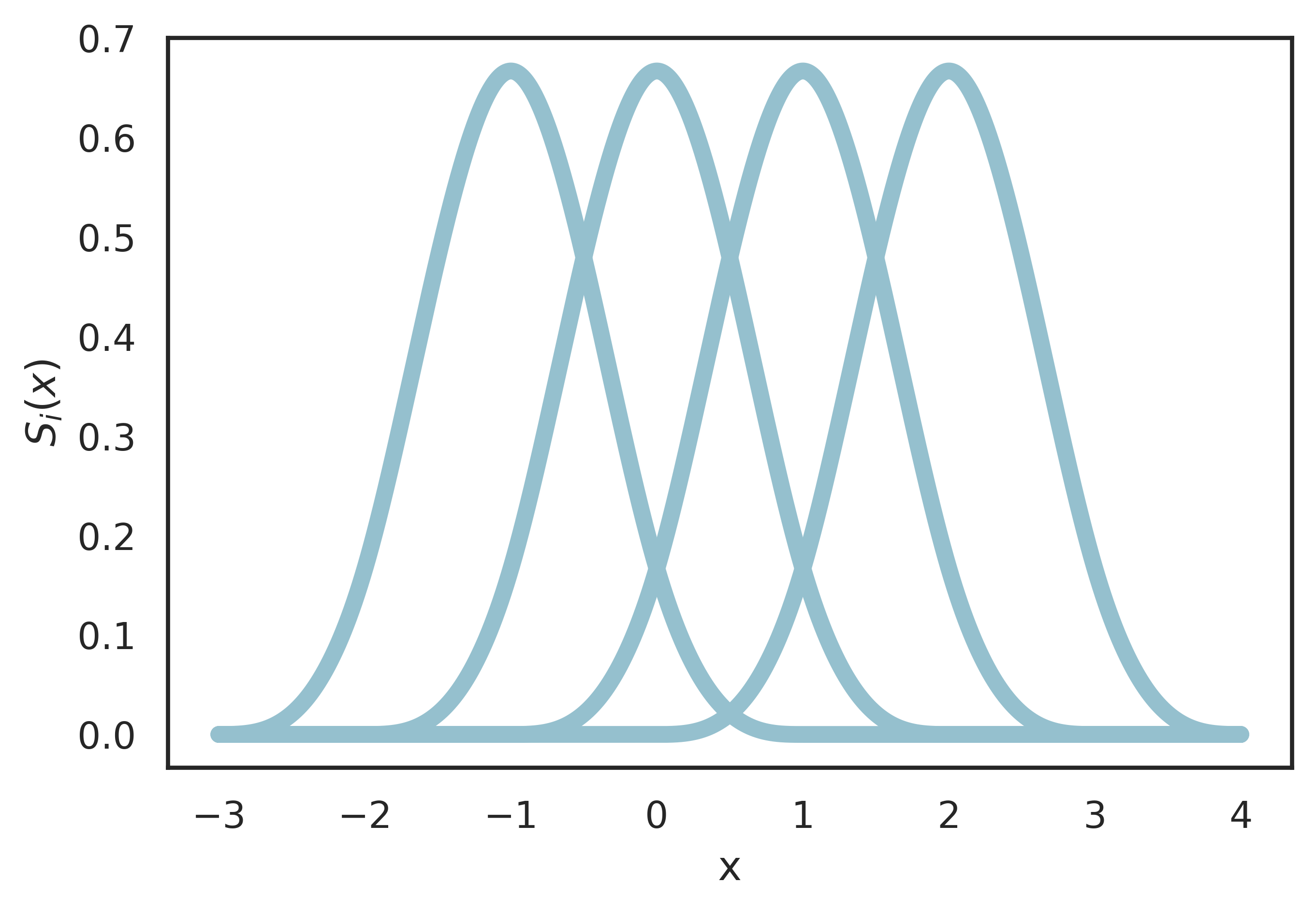}
         \caption{$\rho = 0$}
         \label{fig:basis_density_0}
     \end{subfigure}
     \hfill
     \begin{subfigure}[!h]{0.32\textwidth}
         \centering
         \includegraphics[width=\textwidth]{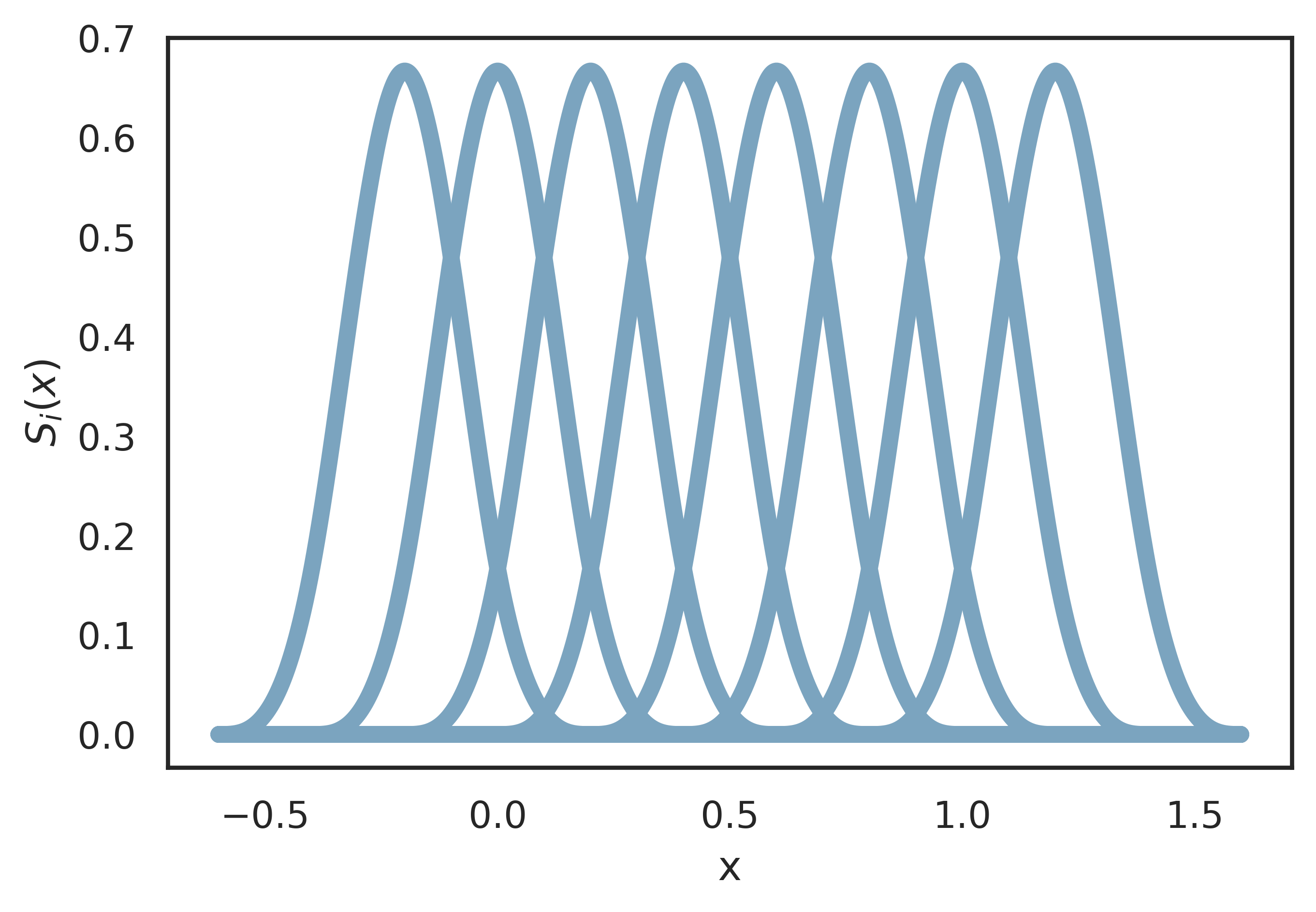}
         \caption{$\rho = 1$}
         \label{fig:basis_density_0}
     \end{subfigure}
        \hfill
     \begin{subfigure}[!h]{0.32\textwidth}
         \centering
         \includegraphics[width=\textwidth]{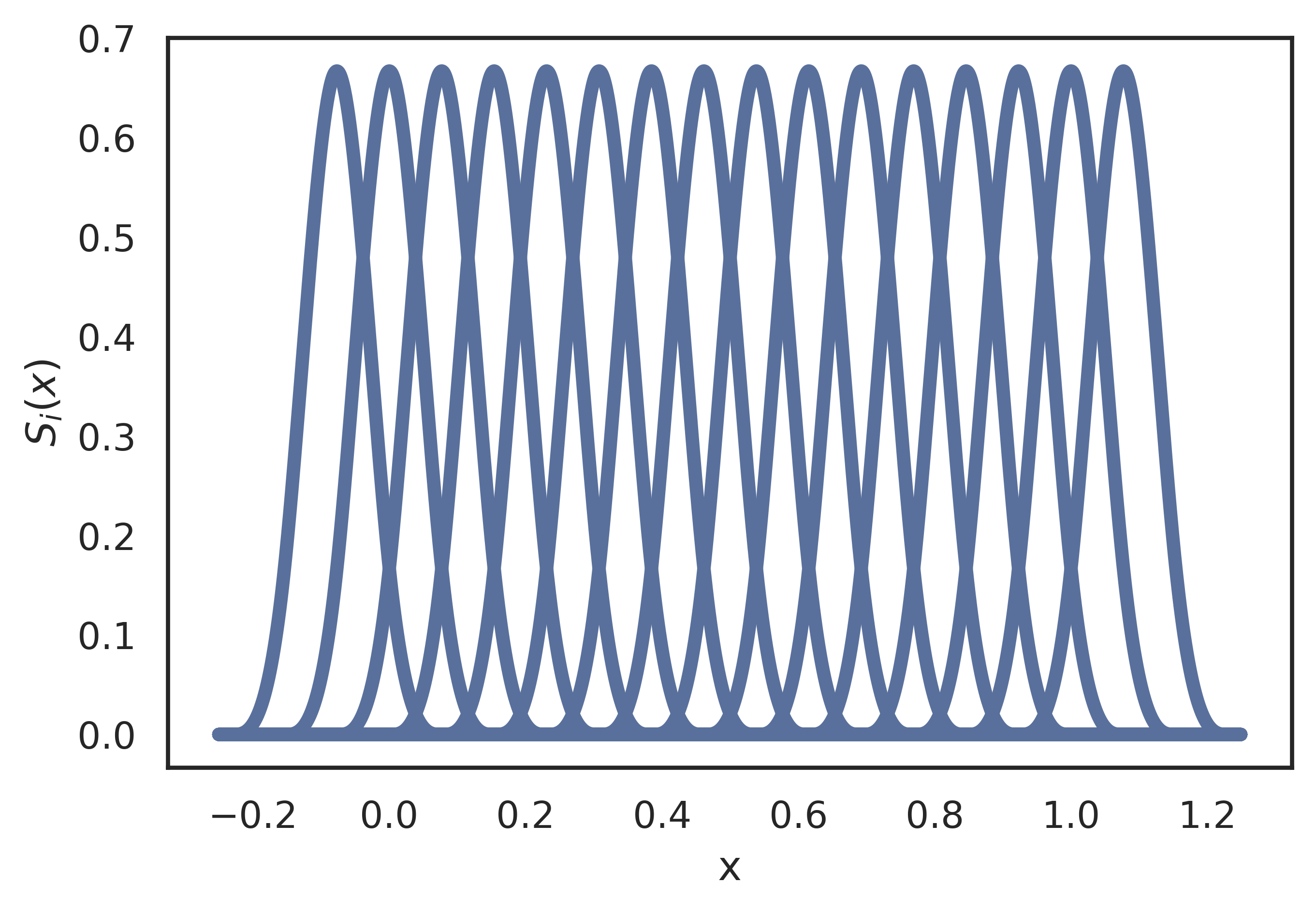}
         \caption{$\rho = 2$}
         \label{fig:basis_density_0}
     \end{subfigure}
        \caption{If uniformly spaced B-splines are used, then each basis function has the same shape. This makes it possible to use the same activation function by scaling and translating the inputs. This is also true for different densities of uniform cubic B-splines.}
        \label{fig:basis_densities_comparison}
\end{figure}

The activation function to implement B-splines is given by:

$$ S(x) =\begin{cases} 
      \frac{1}{6} x^{3} &  0 \leq x < 1\\
      \frac{1}{6} \left[-3(x-1)^{3} +3(x-1)^{2} +3(x-1) + 1 \right] &  1 \leq x < 2\\
      \frac{1}{6} \left[3(x-2)^{3} -6(x-2)^{2} + 4 \right]  & 2 \leq x < 3\\
      \frac{1}{6} ( 4-x ) ^{3} &  3 \leq x < 4\\
      0 & otherwise 
   \end{cases}
$$

The choice was made to use uniform cubic B-splines due to their excellent performance and robustness to catastrophic forgetting, illustrated in Figure~\ref{fig:properties_visual_proof}. Using uniform B-splines instead of arbitrary sub-interval partitions (also called knots in literature) makes optimisation easier. Optimising partitions is non-linear, but optimising only coefficient (also called control points) is linear and thus convex.

\begin{figure} [!h]
     \centering
     \begin{subfigure}[!h]{0.45\textwidth}
         \centering
         \includegraphics[width=\textwidth]{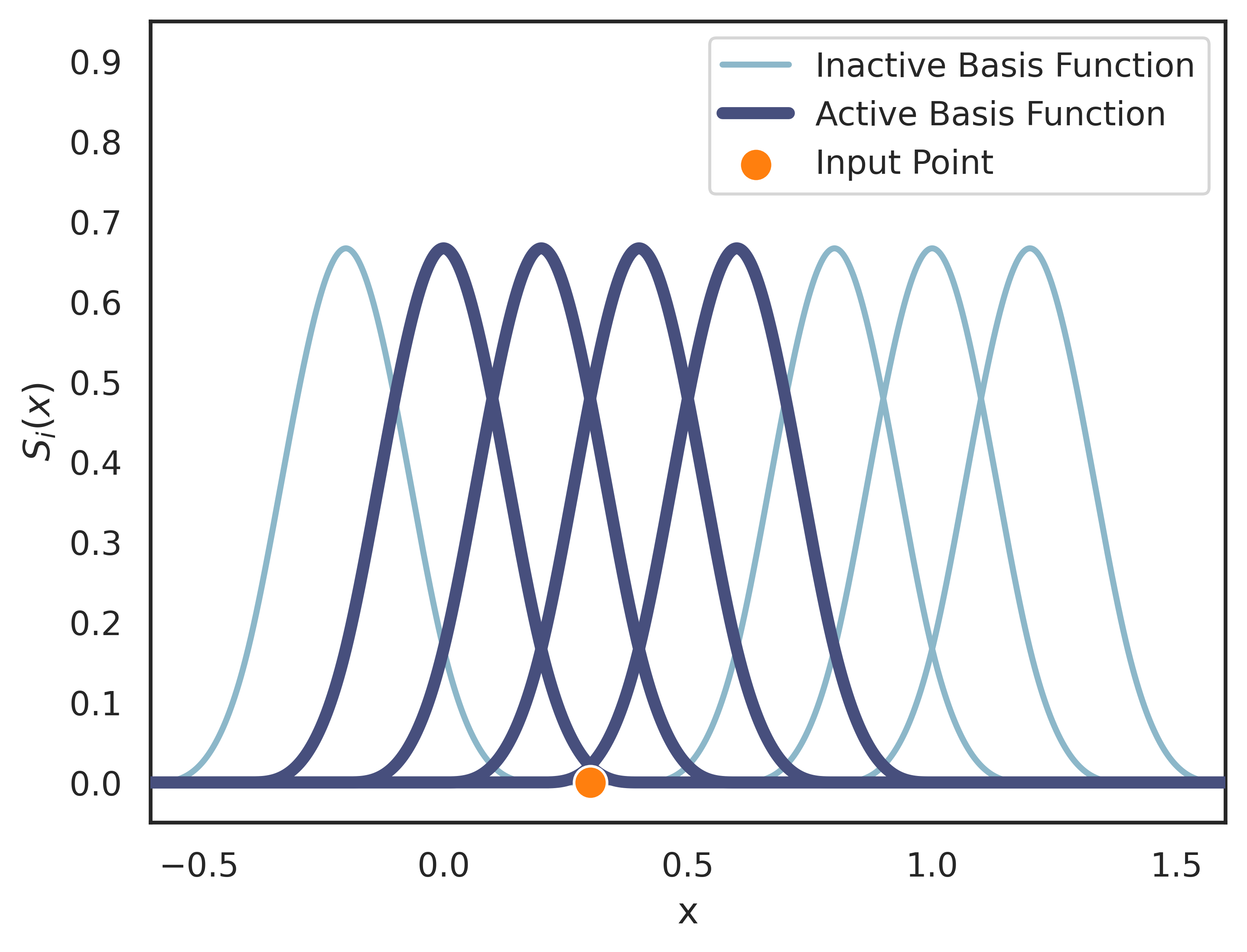}
         \caption{The activation is sparse and bounded.}
     \end{subfigure}
     \hfill
     \begin{subfigure}[!h]{0.45\textwidth}
         \centering
         \includegraphics[width=\textwidth]{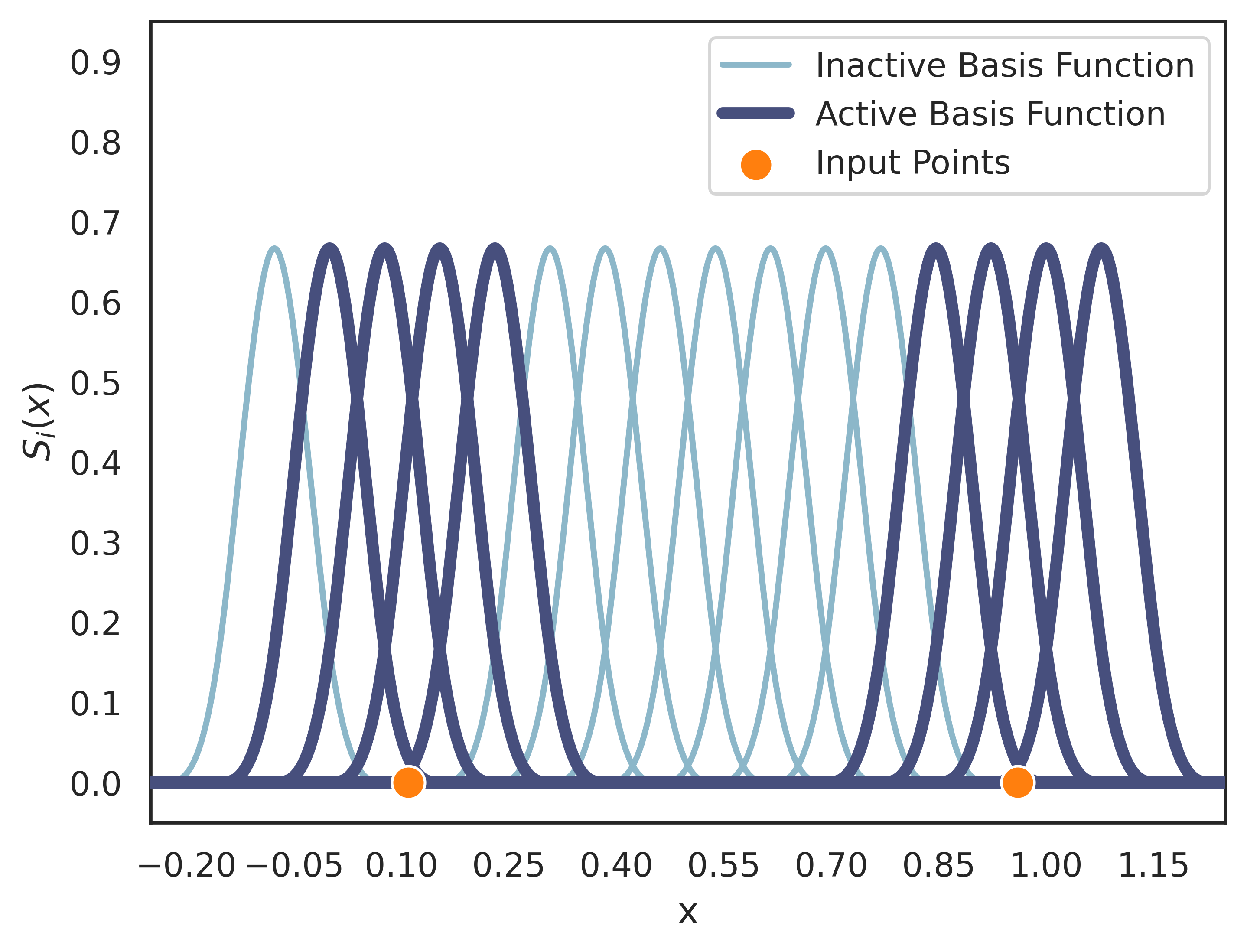}
         \caption{Distal orthogonality.}
     \end{subfigure}
     \caption{Single-variable function}
     \label{fig:properties_visual_proof}
\end{figure}

Each basis function is multiplied by a parameter and summed together. The total number of basis functions is typically fixed. Cubic B-splines are 3\textsuperscript{rd} order polynomials, and thus require a minimum of $3+1=4$ control points or basis functions.

Instead of considering arbitrary densities of uniform cubic B-splines, we look at powers of two times the minimum number of basis functions, called $\rho$-density B-spline functions.

\begin{definition}[$\rho$-density B-spline function]
\label{thm_atlas_def}
A $\rho$-density B-spline function is a uniform cubic B-spline function with $2^{\rho+2}$ basis functions:
\begin{equation*} 
\begin{split}
f(x) 
= \sum_{i=1}^{2^{\rho+2}} \theta_{i} S_{i}(x) 
= \sum_{i=1}^{2^{\rho+2}} \theta_{i} S(w_{i}x + b_{i})
= \sum_{i=1}^{2^{\rho+2}} \theta_{i} S( (2^{\rho+2}-3)x + 4 - i )
\end{split}
\end{equation*}
\end{definition}

Consider the problem of expanding a single-variable function approximator with more basis functions to increase its expressive power. Using the Fourier basis makes it trivially easy by adding higher frequency sines and cosines with coefficients initialised to zero. It is trickier to achieve something similar with uniform cubic B-splines. There are algorithms for creating new splines from existing splines with knot insertion, but the intermediate steps result in non-uniform knots and splines. A simple and practical compromise that we propose is to use mixtures of different $\rho$-density B-spline functions, as illustrated in Figure~\ref{fig:basis_densities_comparison}. 

\begin{definition}[mixed-density B-spline function]
\label{thm_atlas_def}
A mixed-density B-spline function is a single-variable function approximator that is obtained by summing together different $\rho$-density B-spline functions. Only the maximum $\rho$-density B-spline function has trainable parameters, the others are constant. Mixed-density B-spline functions are of the form:
\begin{equation*}
    \begin{aligned}
f(x) 
&= \sum_{\rho=0}^{r} \sum_{i=1}^{ 2^{\rho+2}} \theta_{\rho,i} S_{\rho,i}(x) 
    \end{aligned}
\end{equation*}
\end{definition}

Only the \textbf{maximum} $r=\rho$-density B-spline has trainable coefficients. All lower density $r>\rho$-density B-spline have frozen and constant coefficients. The maximum $r=\rho$-density B-spline has trainable coefficients with gradient updates that are orthogonal if the distance between two inputs is large enough.

Similar to increasing the expressiveness of a Fourier basis function approximator by adding higher frequency terms, one can add larger density cubic B-spline functions. Analytically, we can initialise all the new scalar parameters $\theta_{r+1,i} = 0, \; \forall i \in \bf{N}$ such that:

\begin{equation*}
    \begin{aligned}
            f(x) 
&= \sum_{\rho=0}^{r} \sum_{i=1}^{ 2^{\rho+2}} \theta_{\rho,i} S_{\rho,i}(x)
= \sum_{\rho=0}^{r+1} \sum_{i=1}^{ 2^{\rho+2}} \theta_{\rho,i} S_{\rho,i}(x) 
    \end{aligned}
\end{equation*}

It is therefore possible to create a minimal model with $r=0$ initialised at zero, and train the model until convergence. Then one can create a new model with $r=1$, by subsuming the previous model's parameters, and train this more expressive model until convergence. This process of training and expansion can be continued indefinitely, and is shown in Figure~\ref{fig:fig_before_after_training_density_increase}.
 
\begin{figure}[!h]
\centering
\includegraphics[width=0.995\linewidth]{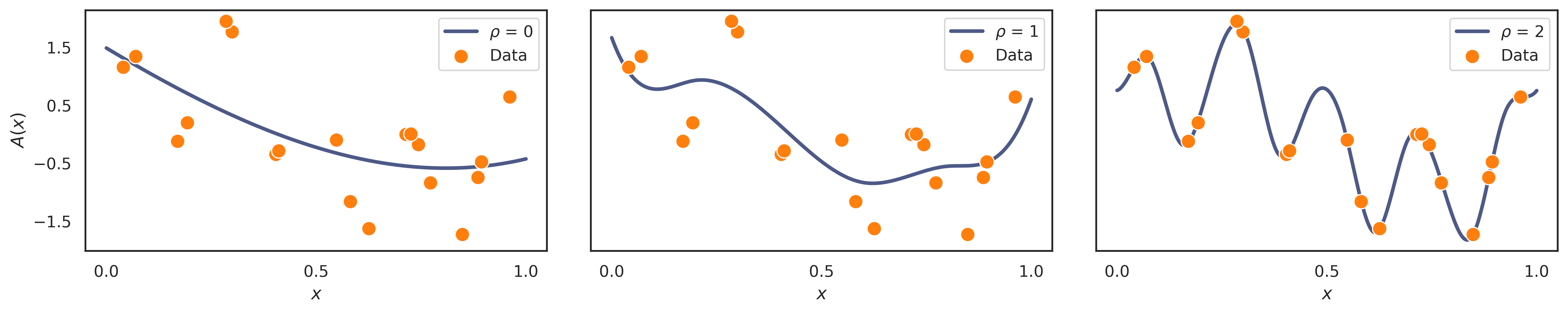}
\caption{Doubling densities of basis functions before and after training.}
\label{fig:fig_before_after_training_density_increase}
\end{figure}

\section{ATLAS}\label{sec:ATLAS}
ATLAS is named for carrying the burden of all it must remember, after the Titan god Atlas in Greek mythology who was tasked with holding the weight of the world. ATLAS is also an acronym for AddiTive exponentiaL Additive Splines.

\begin{definition}[ATLAS]
\label{thm_atlas_def}
ATLAS is a function approximator of $n$ variables, with mixed-density B-spline functions $f_{j}(x_{j})$, $g_{i,j}(x_{j})$, and $h_{i,j}(x_{j})$ in the form:

\begin{equation*} 
\begin{split}
A(\vec{\mathbf{x}}) 
 := \sum_{j=1}^{n} f_{j}(x_{j}) + 
\sum_{k = 1}^{M} \frac{1}{k^{2}} \exp( \Sigma_{j=1}^{n} g_{k,j}(x_{j})) 
                - \frac{1}{k^{2}} \exp( \Sigma_{j=1}^{n} h_{k,j}(x_{j})) 
\end{split}
\end{equation*}
\end{definition}

ATLAS is equivalently given by the compact notation:

\begin{equation*} 
\begin{split}
A(\vec{\mathbf{x}}) 
:= F(\vec{\mathbf{x}}) + 
\sum_{k = 1}^{M} \frac{1}{k^{2}} \exp(G_{k}(\vec{\mathbf{x}})) 
                - \frac{1}{k^{2}} \exp(H_{k}(\vec{\mathbf{x}}))
\end{split}
\end{equation*}

The absolutely convergent series of scale factors $k^{-2}$ was chosen for numerical stability and to ensure the model is absolutely convergent. Another feature is that the series of scale factors also breaks the symmetry that would otherwise exist if all mixed-density B-spline functions were initialised to zero. Initialising all the parameters to be zero is a departure from the conventional approach of random initialisation. The number of exponential terms can be increased without changing the output of the model. We can choose to initialise $G_{M+1}(\vec{\mathbf{x}}) = 0$ and $H_{M+1}(\vec{\mathbf{x}}) = 0$, such that the model capacity can be increased at will.

ATLAS is a universal function approximator with some inherent memory retention. It possesses three properties atypical of most universal function approximators:
\begin{enumerate}
    \item The activity within ATLAS is sparse -- most neural units are zero and inactive.
    \item  The gradient vector with respect to trainable parameters is bounded regardless of the size and capacity of the model, so training is numerically stable for many possible training hyper-parameters.
    \item Inputs that are sufficiently far from each other have orthogonal representations.
\end{enumerate} 

The proofs of the three properties follows from the single-variable case, the assumption of bounded single-variable functions and parameters, and the absolutely convergent $k^{-2}$ scale factors.

\begin{property}[Sparsity]
For any $\vec{\mathbf{x}} \in D(A) \subset R^{n}$ and bounded trainable parameters $\theta_{i}$ with index set $\Theta$, the gradient vector of trainable parameters (for ATLAS) is sparse:

$$ 
\norm{ \grad_{\vec{\mathbf{\theta}}} A(\vec{\mathbf{x}} )}_{0} 
= \sum_{i \in \Theta}  d_{Hamming} \left(\frac{\partial A}{\partial \theta_{i}} (\vec{\mathbf{x}}),0  \right)
\leq 4 n (2M+1)
$$
\end{property}

\begin{remark}
For a fixed number of variables $n$, the model has a total of $n2^{r+2}(2M+1)$ trainable parameters. The gradient vector has a maximum of $4 n (2M+1)$ non-zero entries, which is independent of $r$. Recall that only the maximum density ($\rho=r$) cubic B-spline function has trainable parameters. The fraction of trainable basis functions that are active is at most $2^{-r}$. Sparsity entails efficient implementation, and suggests possible memory retention and robustness to catastrophic forgetting.
\end{remark}

\begin{property}[Gradient flow attenuation]
For any $\vec{\mathbf{x}} \in D(A) \subset R^{n}$ and bounded trainable parameters $\theta_{i}$ with index set $\Theta$: if all the mixed-density B-spline functions are bounded, then the gradient vector of trainable parameters for ATLAS is bounded:

$$ 
\norm{ \grad_{\vec{\mathbf{\theta}}} A( \vec{\mathbf{x}} )}_{1} 
= \sum_{i \in \Theta}  \left| \frac{\partial A}{\partial \theta_{i}} ( \vec{\mathbf{x}} ) \right| 
< U
$$
\end{property}

\begin{remark}
For a fixed number of variables $n$, the model has a total of $n2^{r+2}(2M+1)$ trainable parameters. The factor of $k^{-2}$ inside the expression for ATLAS is necessary to ensure the sum is convergent in the limit of infinitely many exponential terms $M \to \infty$. Only the maximum density ($\rho=r$) cubic B-spline function has trainable parameters, so that the gradient vector is bounded in the limit of arbitrarily large densities $r \to \infty$. Smaller densities cannot be trainable, otherwise this property does not hold. The bounded gradient vector implies that ATLAS is numerically stable during training, regardless of its size or parameter count.
\end{remark}

\begin{property}[Distal orthogonality]
For any $\vec{\mathbf{x}},\vec{\mathbf{y}} \in D(A) \subset R^{n}$ and bounded trainable parameters $\theta_{i}$ for an ATLAS model $A(\vec{\mathbf{x}})$:

$$
 \min_{j=1, \dots , n }
\{ |x_{j} - y_{j}| \} > 2^{-r} 
\implies  
\langle
\grad_{\vec{\mathbf{\theta}}} A(\vec{\mathbf{x}}) 
, 
\grad_{\vec{\mathbf{\theta}}} A(\vec{\mathbf{y}})
\rangle
= 0
$$
\end{property}

\begin{remark}
Two points that sufficiently differ in each input variable have orthogonal parameter gradients. 
Distal orthogonality means ATLAS is reasonably robust to catastrophic forgetting, without other regularisation and training techniques. However, memory retention can still potentially be improved when used in conjunction with other techniques.
\end{remark}

ATLAS can be implemented with 1D convolution, reshaping, embedding, multiplication and dense layers. The same basis functions have to be computed for each input variable, hence 1D convolutions. By correctly scaling, shifting, and rounding inputs one can compute only the non-zero basis functions with embedding layers. The number of basis functions are chosen from powers of two for convenience, with the maximum density B-spline function having exactly $\lambda=4\times2^{r}$ basis functions. Summing over all densities the total number of all basis functions in each input variable is at most $2 \lambda$, because a geometric series was used. For every output dimension $p$, there are $2M$ exponentials. Each exponential has $n$ single variable functions, with at most $2\lambda$ cubic B-spline basis functions each. ATLAS models have time complexity $\mathcal{O}(pMn\log\lambda )$, and $\mathcal{O}(pMn\lambda)$ space complexity.

\section{Methodology}

The 1-,2- and 8-dimensional models were considered for evaluation, in combination with a chosen width for the update region in Task 2 from $0.1$ to $0.9$ in $0.1$ increments. 30 trials were performed for each combination of model dimension and update region width. Mean Absolute Error (MAE) loss function, the Adam optimiser, and mini batch sizes of 100 are used throughout all experiments.

At the beginning of each trial (for a given dimension and update region width) a random learning rate was sampled uniformly between $10^{-6}$ and $0.01+10^{-6}$. A random noise level was sampled from an exponential distribution with scale parameter equal to one. The Task 1 target function is constructed from 1000 Euclidean radial basis functions (RBFs) with locations chosen uniformly over the entire input domain, with RBF scale parameters sampled independently from an exponential distribution (scale parameter equal to 10). The weights of each radial basis function are sampled from a normal distribution with mean zero and standard deviation equal to one. The Task 2 target function is exactly the same as the Task 1 target function -- except for a square-like region with width equal to update region width. The location of the update region is chosen uniformly at random, and such that it is completely inside the domain of the model. The updated region masks the Task 1 target function and instead replaces the values inside it with another function that is sampled from the same distribution as the Task 1 target function, but independently from the Task 1 target function.

After the generation of the target functions 10000 data points are sampled for training, validation, and test sets for Task 1 and Task 2. To simulate the effect of learning unrelated tasks, the training data for Task 2 is only sampled from update region - with no training data outside of it being presented again, by contrast the validation and test sets for Task 2 were sampled over the entire input domain. Gaussian noise with standard deviation equal to the randomly chosen noise level is added to all training data. 
An ATLAS model ($M=10$ positive and $M=10$ negative exponential functions, maximum basis function density $r=4$) with guaranteed distal orthogonality is trained and evaluated on Task 1 and Task 2. Then a modified ATLAS model ($M=10$ positive and $M=10$ negative exponential functions, maximum basis function density $r=4$, trainable lower density basis functions) without guaranteed distal orthogonality is trained and evaluated on Task 1 and Task 2 using the same data sets as previously mentioned model. The final test errors for Task 2 are presented. A randomly selected trial of the 2-dimensional case is shown for visual inspection. The experiments presented in the main body of the paper were performed on Google Colab and the relevant code is provided.

\section{Results}
        
As shown in Figure~\ref{fig:visualisation_experiment_2dim} the effect of distal orthogonality is clear and crisp boundaries that limit the effect of Task 2 on the memory of Task 1. Without distal orthogonality there are more off-target effects that can be visualised.

\begin{figure} [!h]
     \centering
     \begin{subfigure}[!h]{0.90\textwidth}
         \centering
         \includegraphics[width=0.90\linewidth]{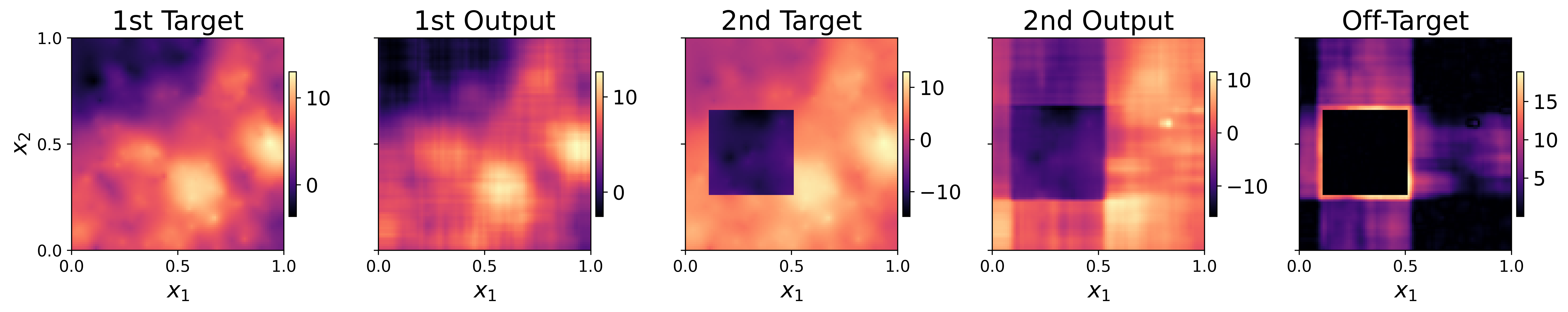}
         \caption{Guaranteed distal orthogonality, Off-target effects deviate from Task 2 target.}
     \end{subfigure}
     \hfill
     \centering
     \begin{subfigure}[!h]{0.90\textwidth}
         \centering
         \includegraphics[width=0.90\linewidth]{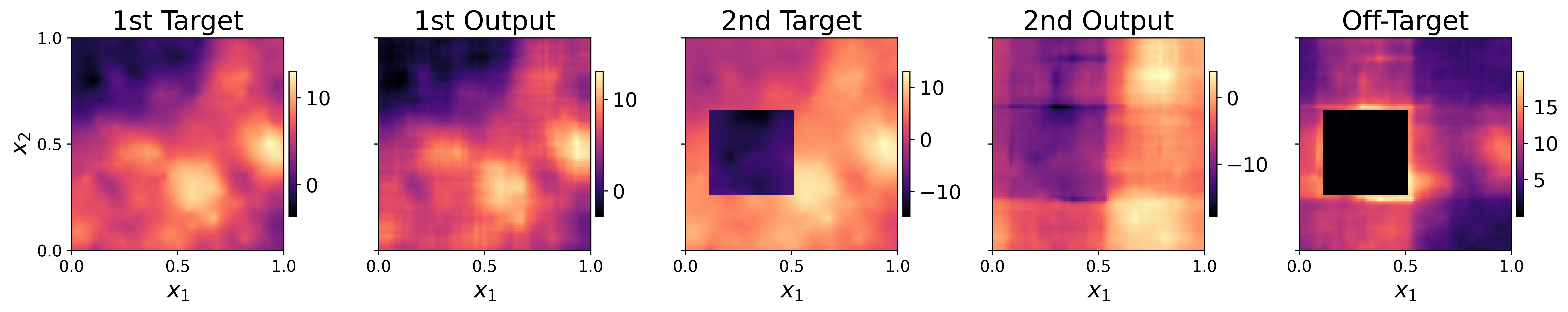}
         \caption{No guaranteed distal orthogonality, Off-target effects deviate from Task 2 target.}
     \end{subfigure}
     \hfill
        \caption{A randomly chosen trial is presented for visual inspection.}
        \label{fig:visualisation_experiment_2dim}
\end{figure}

\begin{figure}[!h]
\centering
\includegraphics[width=0.95\linewidth]{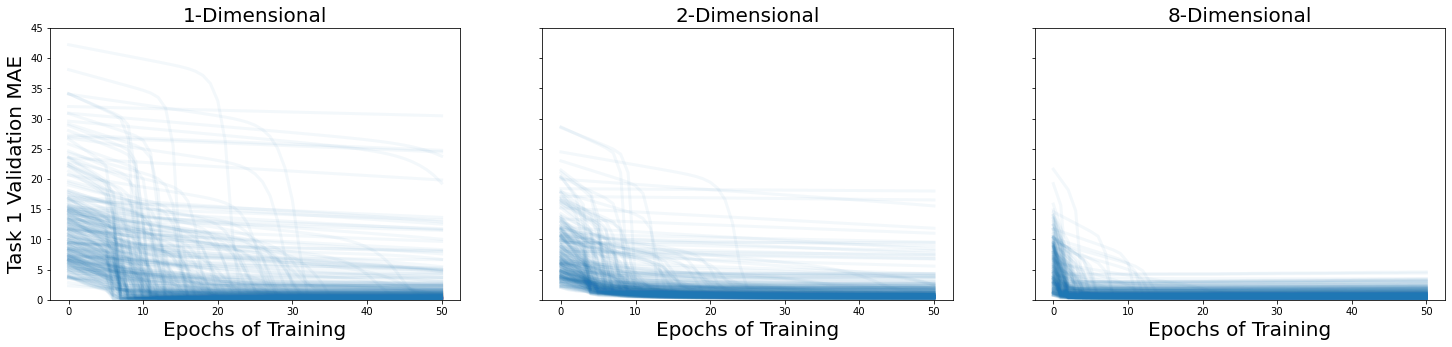}
\caption{Distal orthogonality guaranteed: All validation MAE curves for Task 1.}
\label{fig:fig_before_after_training_density_increase}
\end{figure}

\begin{figure}[!h]
\centering
\includegraphics[width=0.95\linewidth]{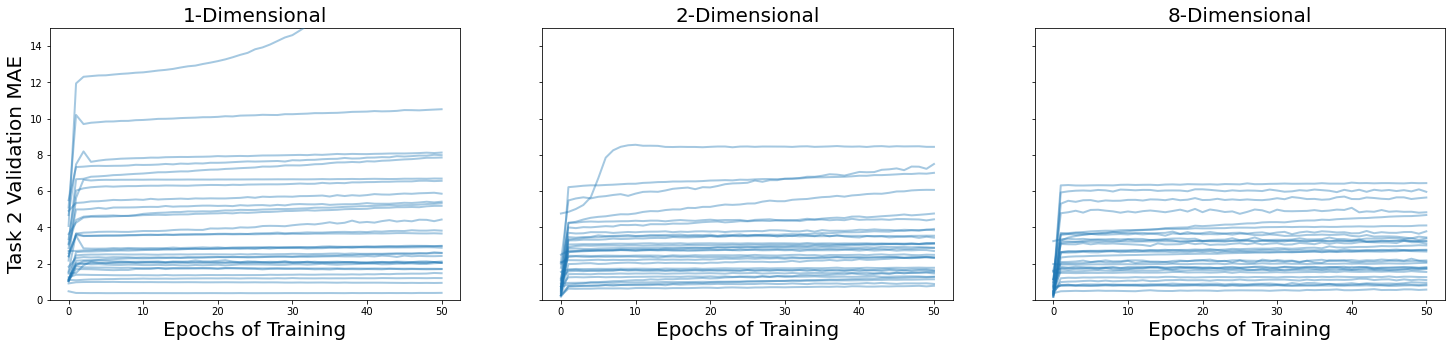}
\caption{No distal orthogonality: Task 2 validation MAE with update region width $\delta=0.1$.}
\label{fig:fig_before_after_training_density_increase}
\end{figure}

\begin{figure}[!h]
\centering
\includegraphics[width=0.95\linewidth]{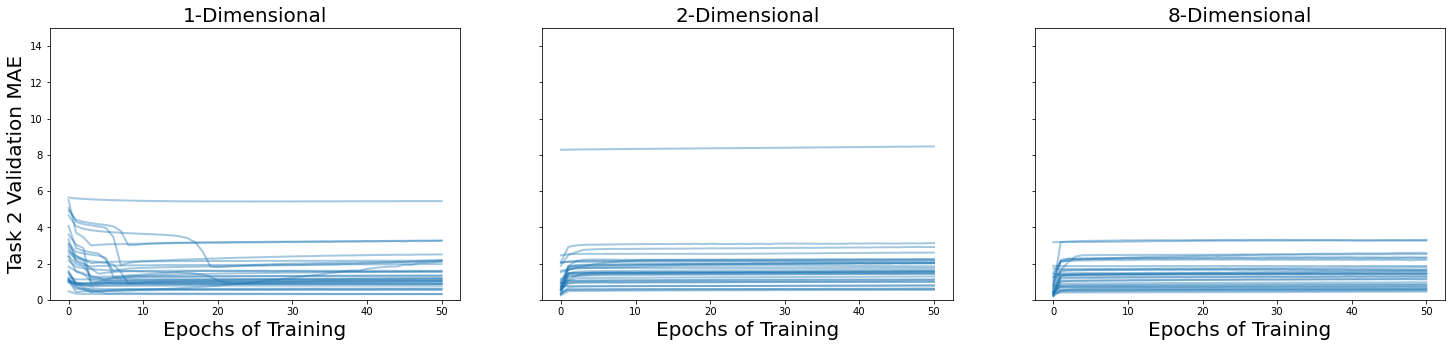}
\caption{Distal orthogonality guaranteed: Task 2 validation MAE with update region width $\delta=0.1$.}
\label{fig:fig_before_after_training_density_increase}
\end{figure}

The effect of distal orthogonality on the averaged MAE for various trials for 1-,2- and 8-dimensional problems are presented as scatter plots of the averaged MAE over 30 trials for different update region widths as shown in Figure~\ref{fig:aggregate_results}. The expected off-target error depends on the dimension of the problem and the width of the updated regions.
   
\begin{figure} [!h]
     \centering
     \begin{subfigure}[!h]{0.3\textwidth}
         \centering
         \includegraphics[width=\textwidth]{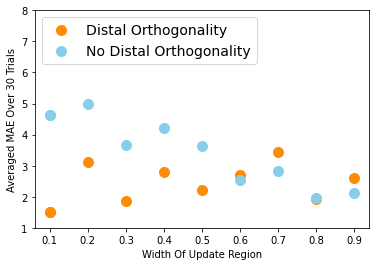}
         \caption{The 1-dimensional model.}
     \end{subfigure}
     \hfill
     \begin{subfigure}[!h]{0.3\textwidth}
         \centering
         \includegraphics[width=\textwidth]{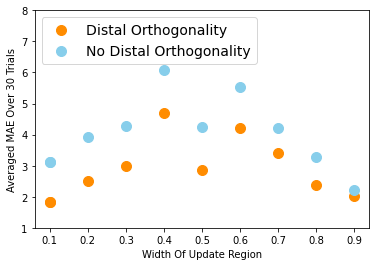}
         \caption{The 2-dimensional model.}
     \end{subfigure}
     \hfill
     \begin{subfigure}[!h]{0.3\textwidth}
         \centering
         \includegraphics[width=\textwidth]{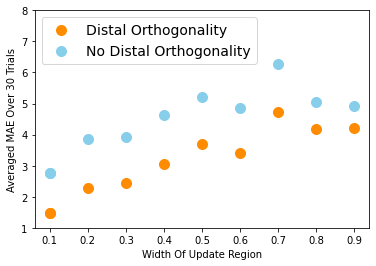}
         \caption{The 8-dimensional model.}
     \end{subfigure}
     \caption{The effect of distal orthogonality on the final test error on task 2 for the 1-,2- and 8-dimensional input.}
     \label{fig:aggregate_results}
\end{figure}

Analytical results to the expected off-target error require simplification, but a reasonable assumption in the absence of other evidence is that each input dimension has equal contribution on the unit hyper-cube. Assume for a fixed input dimension $n$ and some region of width $0<\delta < 1$ where the target function $Y$ is changed such that $\abs{\Delta Y}=1$ is one larger than it was originally. The expected off-target error depends on $k$ the number of input variables inside the updated region: $\varepsilon_{k} \approx \frac{n-k}{n}$. To correctly account for all permutations with the same magnitude of change:

\begin{equation*} 
\begin{split}
p(\varepsilon_{k})
& = \binom{n}{k}  \delta^{n-k} \left( 1 - \delta \right)^{k} \\
\end{split}
\end{equation*}

One can calculate expected change values:

\begin{equation*} 
\begin{split}
\mathbb{E}[\varepsilon]
& = \sum^{n}_{k=0}
\varepsilon_{k} p(\varepsilon_{k})
 \approx \sum^{n}_{k=0} \left( \frac{n-k}{n} \right) \binom{n}{k}  \delta^{n-k} \left( 1 - \delta \right)^{k}
 = \delta \\
\end{split}
\end{equation*}

However if one assumes that the target function inside the updated region of width $\delta$ is correct, with probability $\delta^{n}$ of sampling from the entire input-domain, then the expected off-target error should be:

\begin{equation*} 
\begin{split}
\text{Expected off-target error}
& \approx \delta - \delta^{n} \\
\end{split}
\end{equation*}

This seems consistent with some of the experimental results, but further investigation is needed.

\section{Conclusion}
The main contribution of the paper is theoretical and technical. A representation theorem is presented that outlines how to approximate multi-variable functions with single-variable functions (splines and exponential functions). ATLAS approximates all arbitrary single-variable functions with mixtures of B-spline functions. ATLAS is constructed in such a way that the gradient vector with respect to trainable parameters is bounded, regardless of how large an ATLAS model is. The activation of units in ATLAS is sparse, and allowed for an efficient implementation that only computes non-zero activation values with the aid of embedding layers. The gradient update vector with respect to trainable parameters is orthogonal for different inputs as long as the inputs are sufficiently different from each other.

For every output dimension $p$ in an ATLAS model, there are $2M$ exponentials. Each exponential has $n$ single variable functions, with at most $2\lambda$ cubic B-spline basis functions each. ATLAS models have time complexity $\mathcal{O}(pMn\log\lambda )$, and $\mathcal{O}(pMn\lambda)$ space complexity.

ATLAS was shown to exhibit some memory retention, without the assistance of other techniques. This is a good indication of the potential for combining it with other techniques and models for continual learning. The chosen experiments demonstrated the theoretically derived predictions and contrasted two models, incuding a variant of ATLAS without distal orthogonality guarantees.

As far as societal impacts are concerned: It is possible that ATLAS could allow for the creation of more powerful machine learning algorithms, that require less resources to train and deploy. Further testing is needed to make any concrete claim.

\bibliography{atlas}
\appendix
\section{Additional experiments}\label{experiments}

Additional experimentation was performed with specific target functions. Experiment A was performed on a personal laptop with a 7th generation i7 Intel processor and took a few hours to finish thirty trials. The loss function chosen for training and evaluation is the mean absolute error (MAE). The training data set and test set in all experiments had $10 000$ data points, sampled uniformly at random. Gaussian noise with standard deviation $=0.1$ was added to all training and test data target values. The test set was also used as a validation set to quantify the test error during training. All models were trained with a learning rate of $0.01$ with the Adam optimizer. All models and experiments used batch sizes of 100 during training. 

To test memory retention, two tasks, presented to an Atlas model one after the other, were constructed. The details of each task are given below.

\paragraph{Task 1} 

The training and test sets were sampled uniformly from the Task 1 target function over the domain $\left[0.,1. \right]^{2}$, with Gaussian noise added to the target values. The initial Atlas model was instantiated as a two-variable function that maps to a one-dimensional output, with $r=0$ and $M=0$ such that it is a minimally expressive model. The model was evaluated and trained for 30 epochs. After training, the Atlas model was expanded using the built-in methods, such that $r$ is increased by one, and $M$ is increased by two: $r'=r+1$ and $M'=M+2$. This training-expansion process was repeated four times. The output of the model is presented at the end of each expansion iteration. The target functions for Task 1 in each experiment is labeled $Y(\vec{\mathbf{x}})$.

\paragraph{Task 2} 

The test sets were sampled uniformly from the Task 2 target function over the domain $\left[0.,1. \right]^{2}$, with Gaussian noise added to the target values. The training sets were sampled uniformly over the domain $\left[0.45,0.55 \right]^{2}$, and target values of zero with added Gaussian noise. All models were trained for $6$ epochs. The Task 2 target function in each experiment is given by:

$$ Y'(\vec{\mathbf{x}}) =\begin{cases} 
      0 &  0.45 < x_{i} < 0.55 \; \forall i=1,2,...         \\
      Y(\vec{\mathbf{x}}) & \text{otherwise.} 
   \end{cases}
$$

Task 2 effectively tests if a model changes only where new data was presented, with off-target effects leading to larger test MAE.

\subsection{Experiment A}

\subsubsection{Task 1}
Where the radius is measured from the centre of the domain $[0.5,0.5]$, with the analytical expression $r = \sqrt{ (x_{1}-\frac{1}{2})^{2} + (x_{2}-\frac{1}{2})^{2} }$, and $\theta = \tan^{-1}\left((x-\frac{1}{2})^{2},(y-\frac{1}{2})^{2}\right) $. The Task 1 target function of Experiment A is given by:

\begin{equation*} 
\begin{split}
Y_{A} &= Y_{A}(x_{1},x_{2}) =
\sin{(30r+\theta )}+2 
\end{split}
\end{equation*}

\begin{figure}[!h]
\centering
\includegraphics[width=0.75\linewidth]{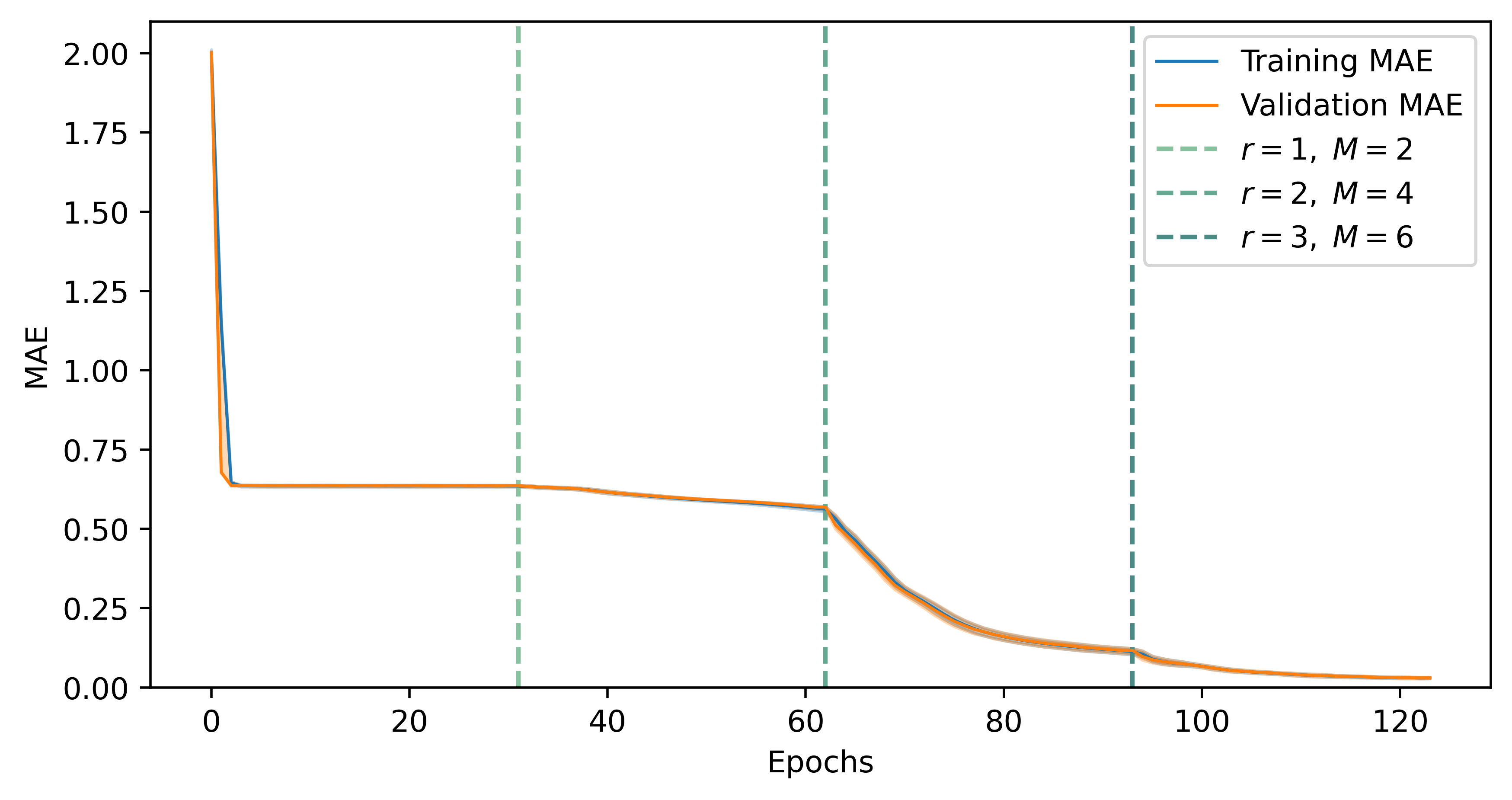} 
\caption{Training and validation MAE during the course of training on Task 1, Experiment A.}
\end{figure}

\begin{figure}[!h]
\centering
\includegraphics[width=0.95\linewidth]{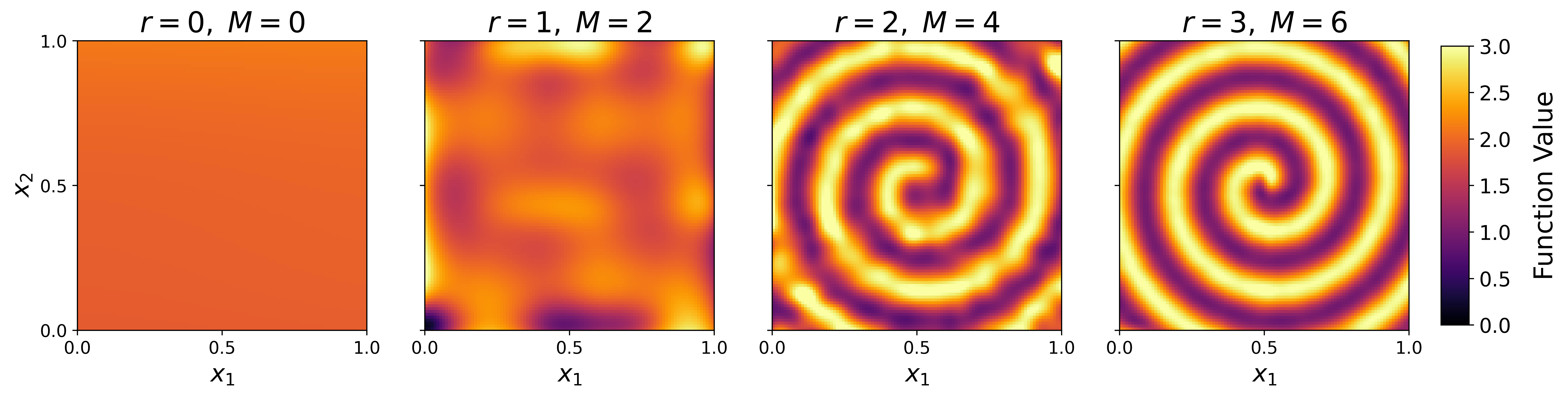}
\caption{Outputs of the model during successive training and expansion iterations, Experiment A.}
\end{figure}

\subsubsection{Task 2} The under-sampled target function $Y'_{A}$ used for validation is given by:

$$ Y'_{A}(x_{1},x_{2}) =\begin{cases} 
      0 &  0.45 < x_{i} < 0.55 \; \forall i=1,2,...         \\
      Y_{A}(x_{1},x_{2}) & \text{otherwise.} 
   \end{cases}
$$

\begin{figure}[!h]
\centering
\includegraphics[width=0.75\linewidth]{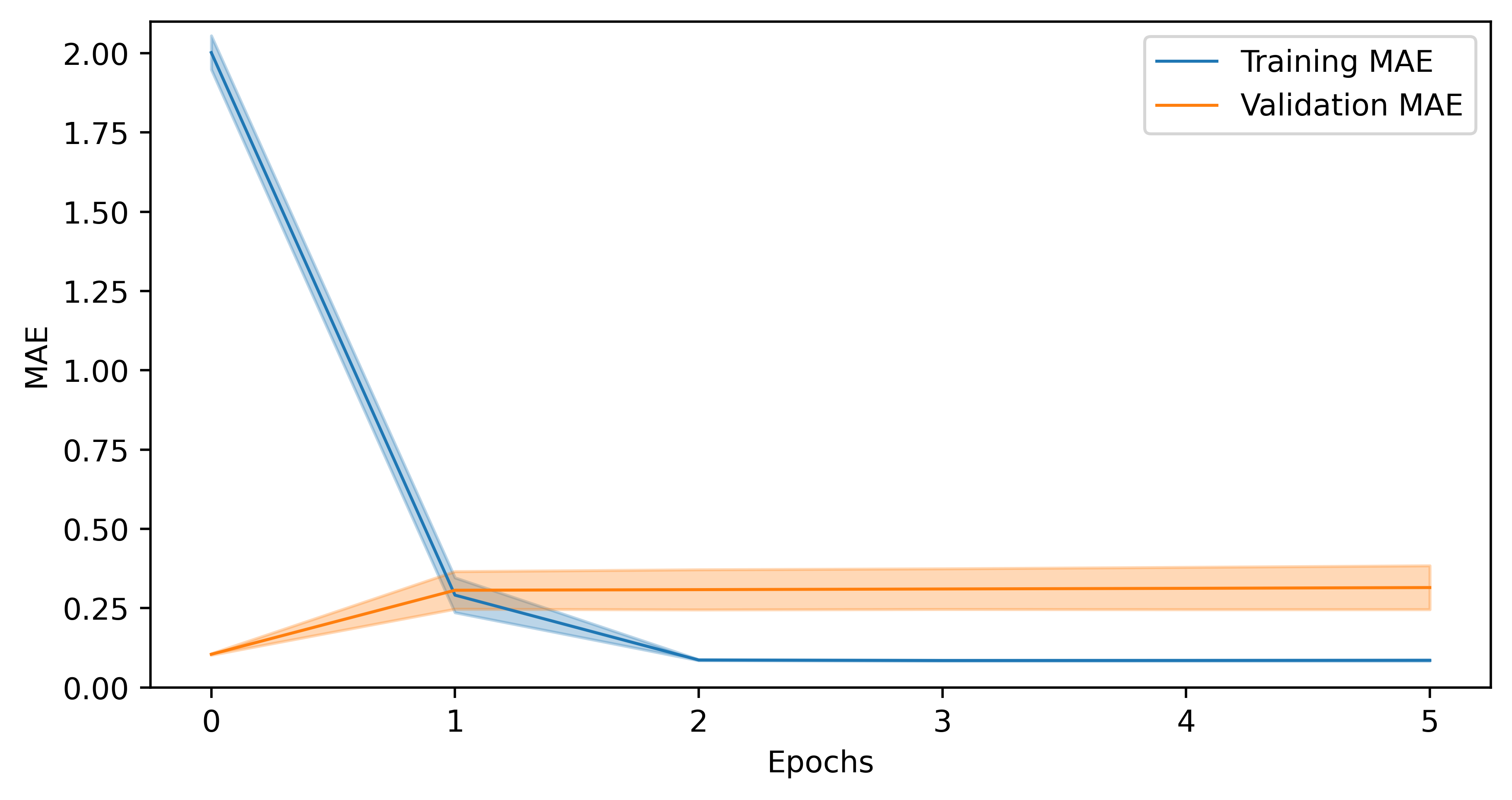}
\caption{Training and validation MAE during the course of training on Task 2, Experiment A.}
\end{figure}

\begin{figure}[!h]
\centering
\includegraphics[width=0.95\linewidth]{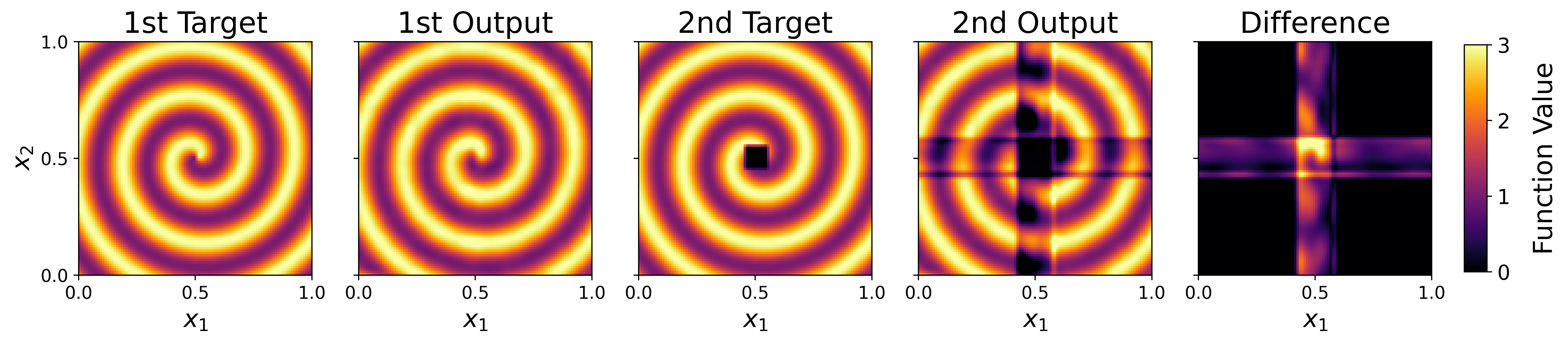}
\caption{Visual inspection of target functions and model outputs over Task 1 and Task 2, Experiment A.}
\end{figure}

\subsection{Experiment B}

\subsubsection{Task 1}
The Task 1 target function for experiment B is given by:

\begin{equation*} 
\begin{split}
Y_{B}(x_{1},x_{2}) &=
\cos^{2}{(20x_{1}-10)}+\cos^{2}{(10x_{2}-5)}
+\exp(-(20x_{1}-10)^{2}-(20x_{2}-10)^{2})
\end{split}
\end{equation*}

\begin{figure}[!h]
\centering
\includegraphics[width=0.75\linewidth]{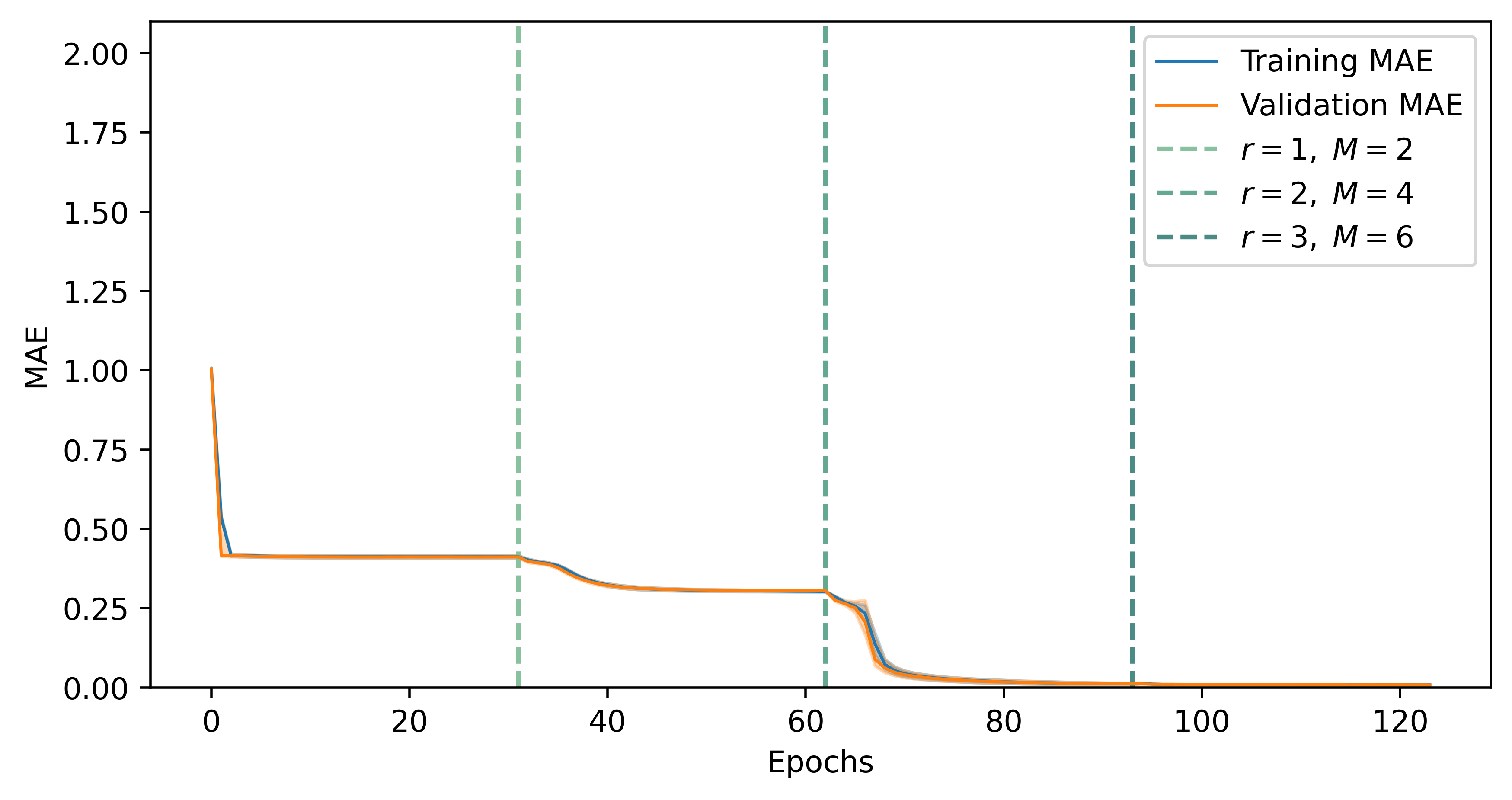} 
\caption{Training and validation MAE during the course of training on Task 1, Experiment B.}
\end{figure}

\begin{figure}[!h]
\centering
\includegraphics[width=0.95\linewidth]{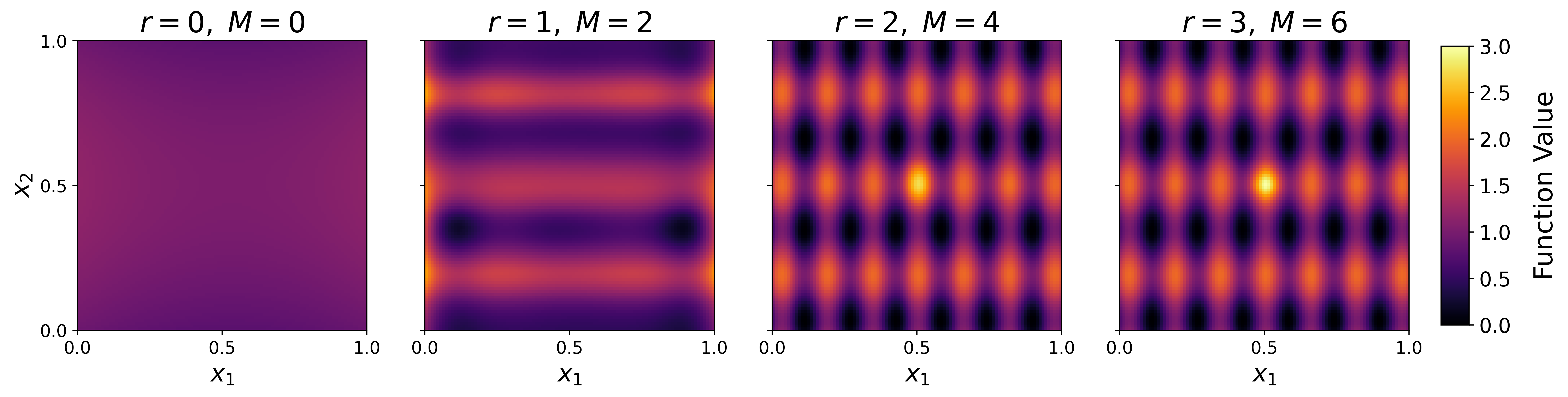}
\caption{Outputs of the model during successive training and expansion iterations, Experiment B.}
\end{figure}

\subsubsection{Task 2} The under-sampled target function $Y'_{B}$ used for validation is given by:

$$ Y'_{B}(x_{1},x_{2}) =\begin{cases} 
      0 &  0.45 < x_{i} < 0.55 \; \forall i=1,2,...         \\
      Y_{B}(x_{1},x_{2}) & \text{otherwise.} 
   \end{cases}
$$

\begin{figure}[!h]
\centering
\includegraphics[width=0.75\linewidth]{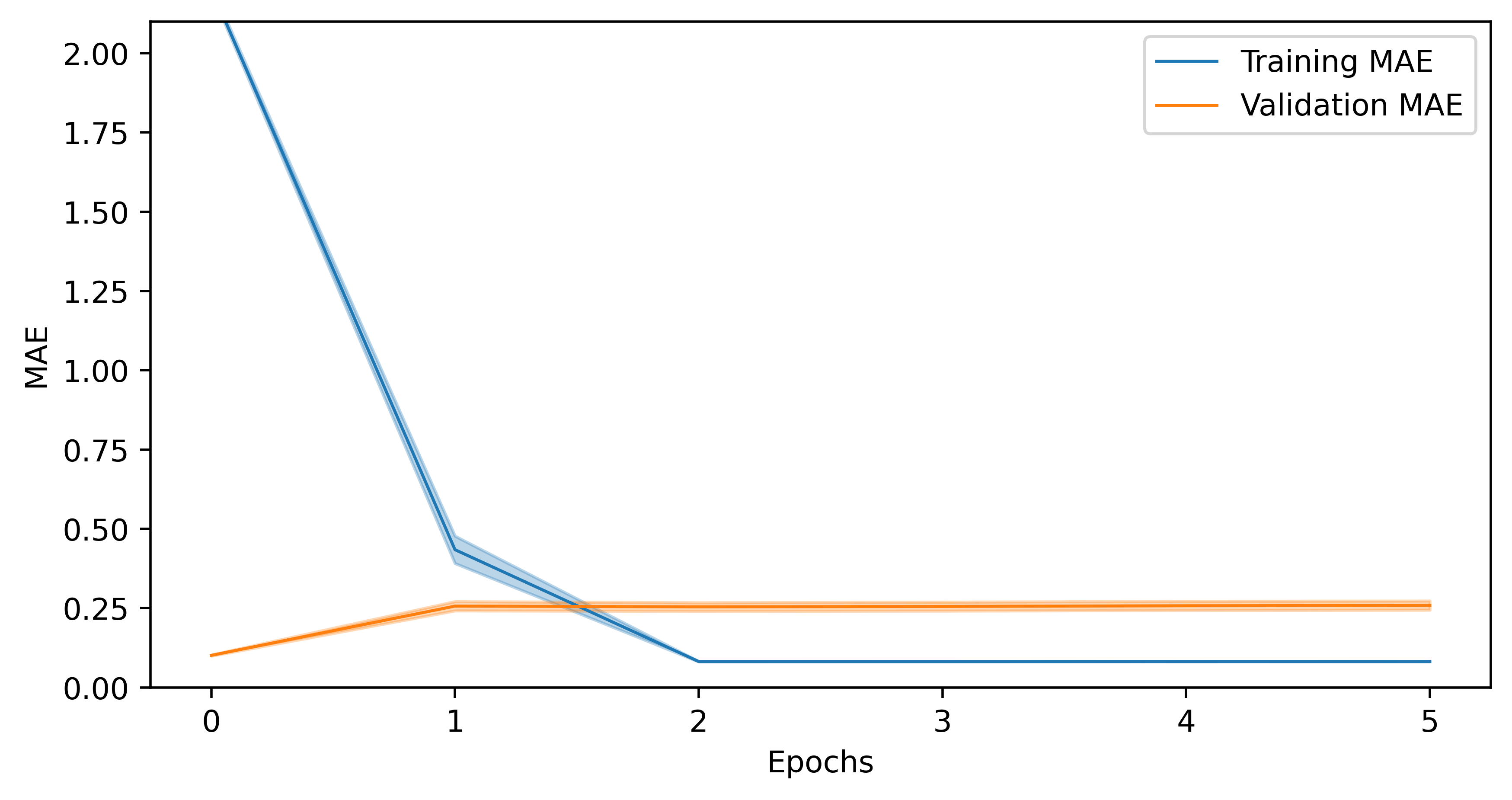}
\caption{Training and validation MAE during the course of training on Task 2, Experiment B}
\end{figure}

\begin{figure}[!h]
\centering
\includegraphics[width=0.95\linewidth]{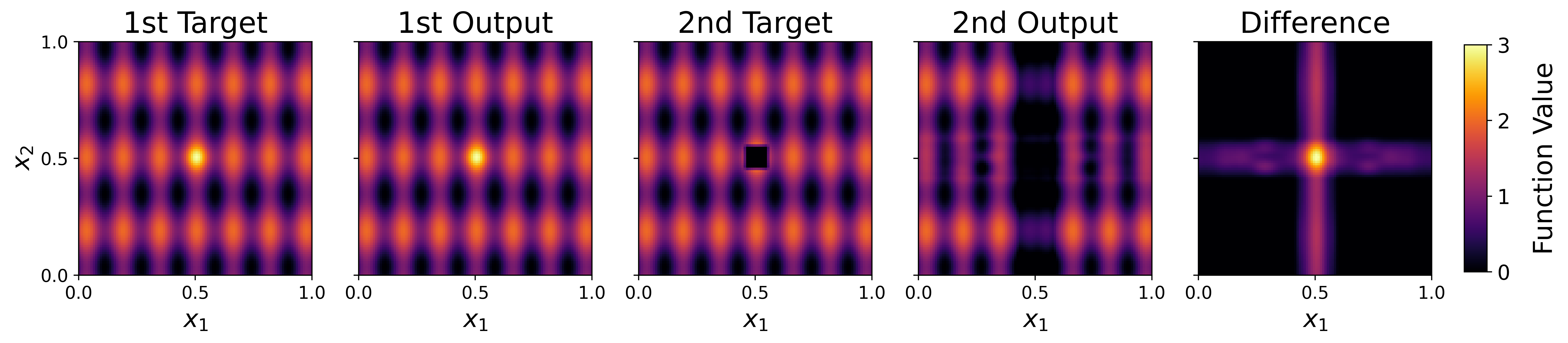}
\caption{Visual inspection of target functions and model outputs over Task 1 and Task 2, Experiment B.}
\end{figure}

\newpage
\subsection{Experiment C}

\subsubsection{Task 1}
The Task 1 target function for Experiment C is given by:

\begin{equation*} 
\begin{split}
Y_{C}(x_{1},x_{2}) &= 2 + \cos{(20x_{1}-10)}\cos{(20x_{2}-10)}
\end{split}
\end{equation*}

\begin{figure}[!h]
\centering
\includegraphics[width=0.75\linewidth]{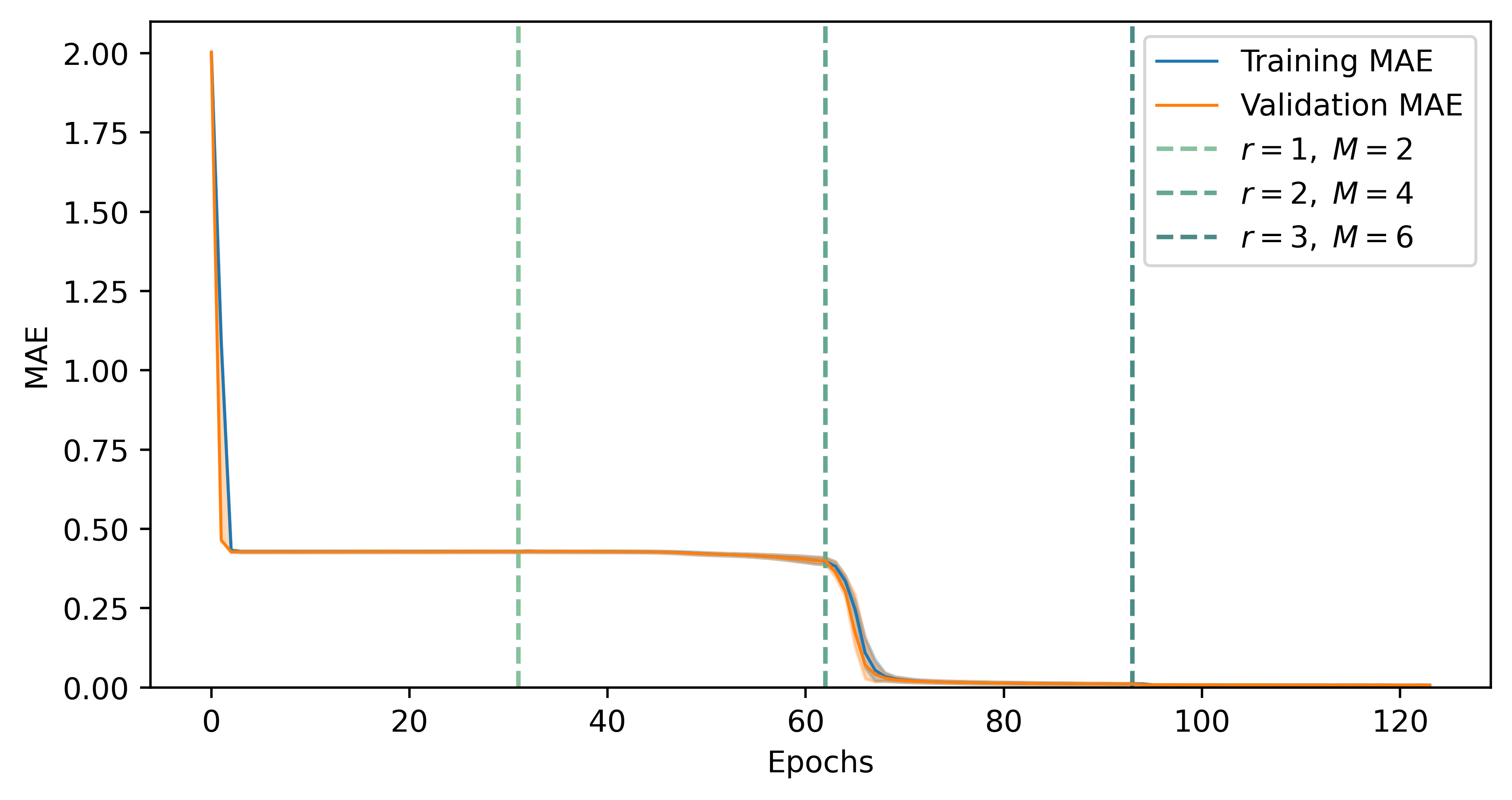} 
\caption{Training and validation MAE during the course of training on Task 1, Experiment C.}
\end{figure}

\begin{figure}[!h]
\centering
\includegraphics[width=0.95\linewidth]{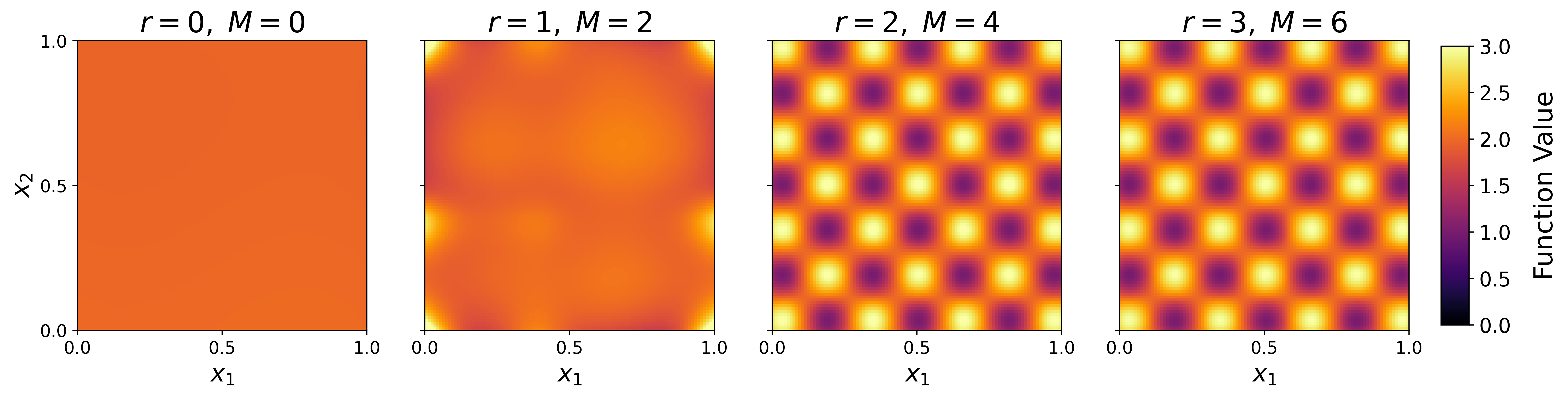}
\caption{Outputs of the model during successive training and expansion iterations, Experiment C.}
\end{figure}

\subsubsection{Task 2} The under-sampled target function $Y'_{C}$ used for validation is given by:

$$ Y'_{C}(x_{1},x_{2}) =\begin{cases} 
      0 &  0.45 < x_{i} < 0.55 \; \forall i=1,2,...         \\
      Y_{C}(x_{1},x_{2}) & \text{otherwise.} 
   \end{cases}
$$

\begin{figure}[!h]
\centering
\includegraphics[width=0.75\linewidth]{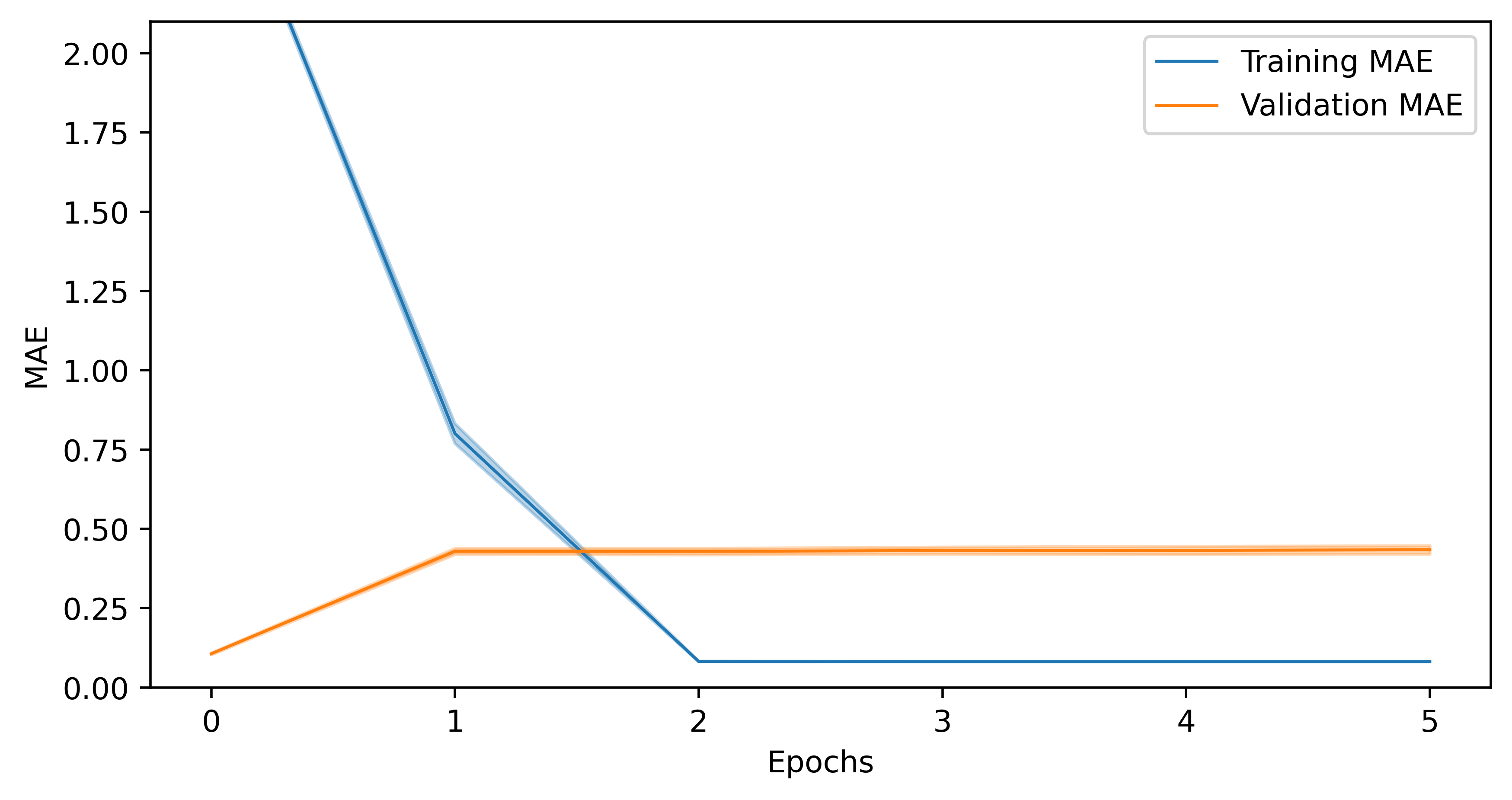}
\caption{Training and validation MAE during the course of training on Task 2, Experiment C.}
\end{figure}

\begin{figure}[!h]
\centering
\includegraphics[width=0.95\linewidth]{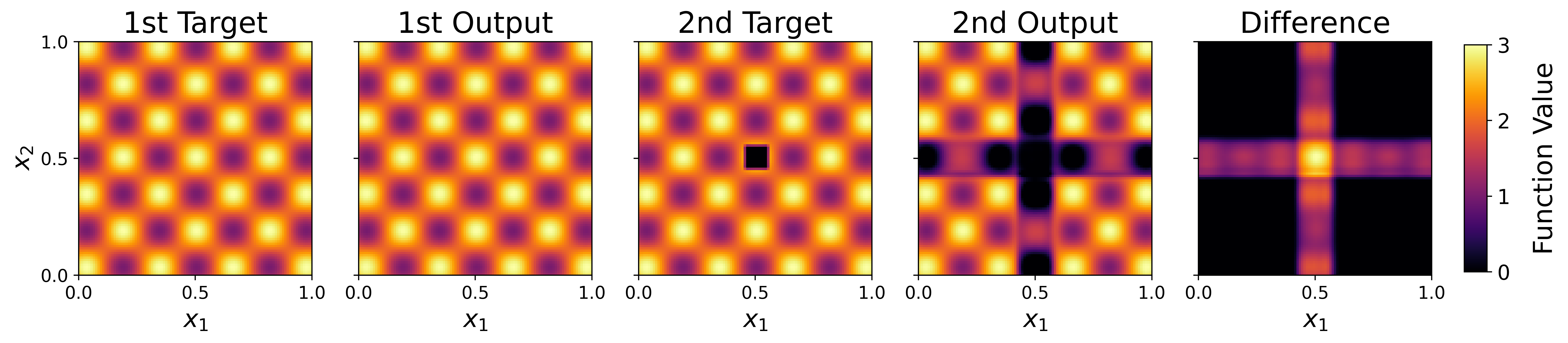}
\caption{Visual inspection of target functions and model outputs over Task 1 and Task 2, Experiment C.}
\end{figure}

\newpage
\subsection{Experiment D}

\subsubsection{Task 1}
Where $\sigma$ is the sigmoid function, the Task 1 target function for experiment D is given by:

\begin{equation*} 
\begin{split}
Y_{D}(x_{1},x_{2}) &= 2 + \sigma(\sin{(2 \pi x_{1})}\sin{(2 \pi x_{2})})
\end{split}
\end{equation*}

\begin{figure}[!h]
\centering
\includegraphics[width=0.75\linewidth]{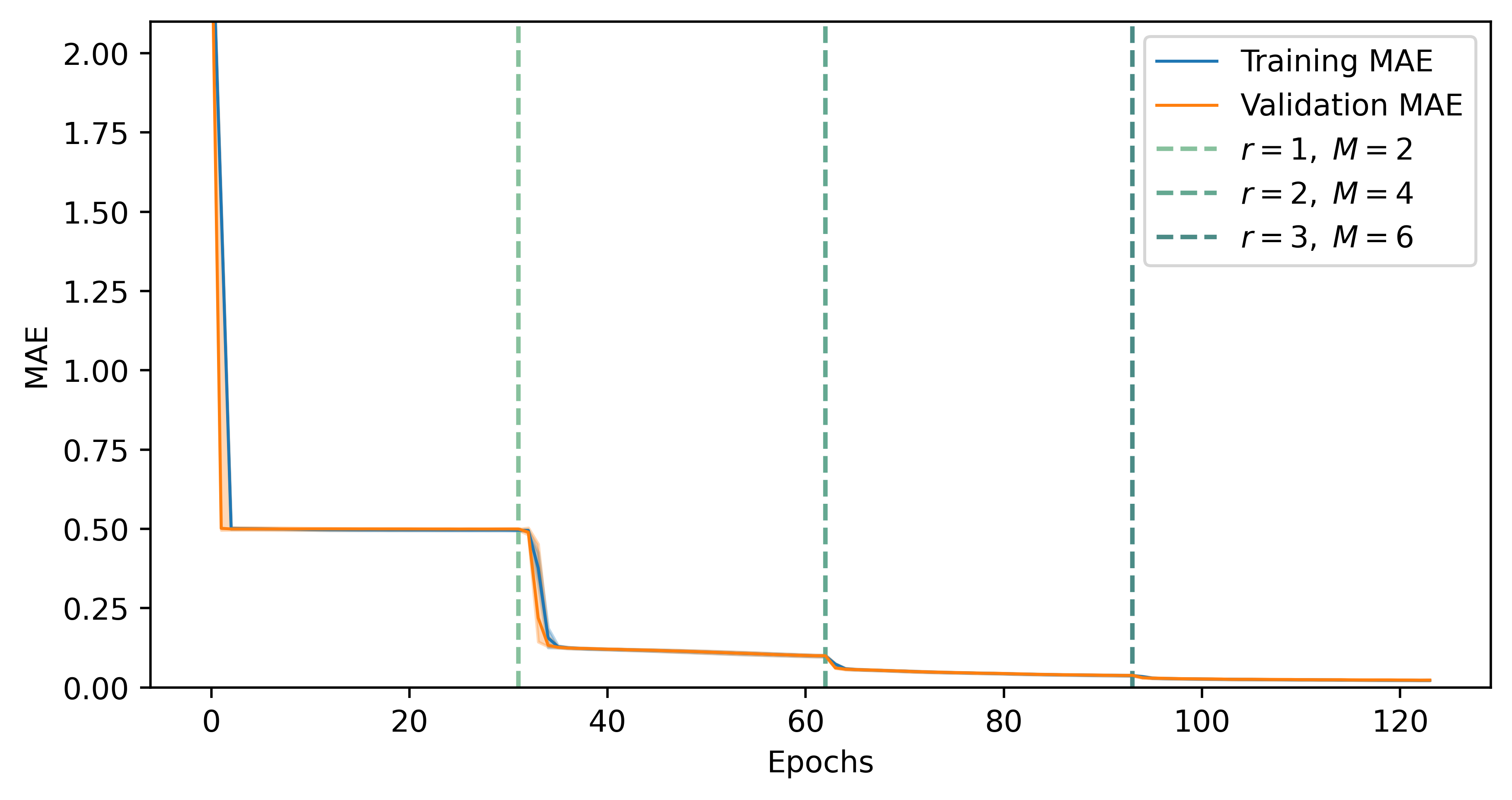} 
\caption{Training and validation MAE during the course of training on Task 1, Experiment D.}
\end{figure}

\begin{figure}[!h]
\centering
\includegraphics[width=0.95\linewidth]{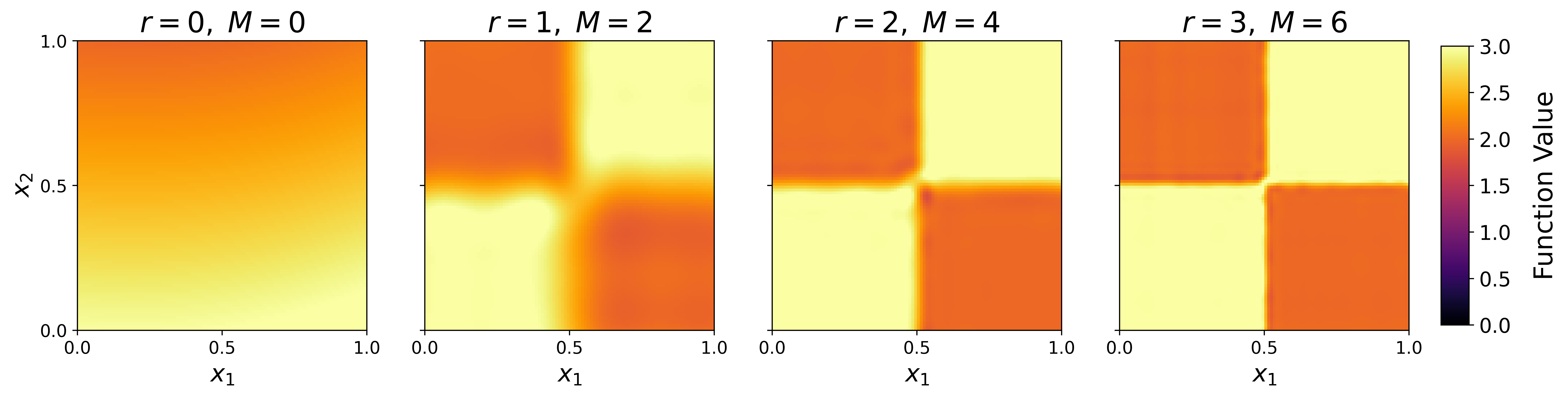}
\caption{Outputs of the model during successive training and expansion iterations, Experiment D.}
\end{figure}

\newpage
\subsubsection{Task 2} The under-sampled target function $Y'_{D}$ used for validation is given by:

$$ Y'_{D}(x_{1},x_{2}) =\begin{cases} 
      0 &  0.45 < x_{i} < 0.55 \; \forall i=1,2,...         \\
      Y_{D}(x_{1},x_{2}) & \text{otherwise.} 
   \end{cases}
$$

\begin{figure}[!h]
\centering
\includegraphics[width=0.75\linewidth]{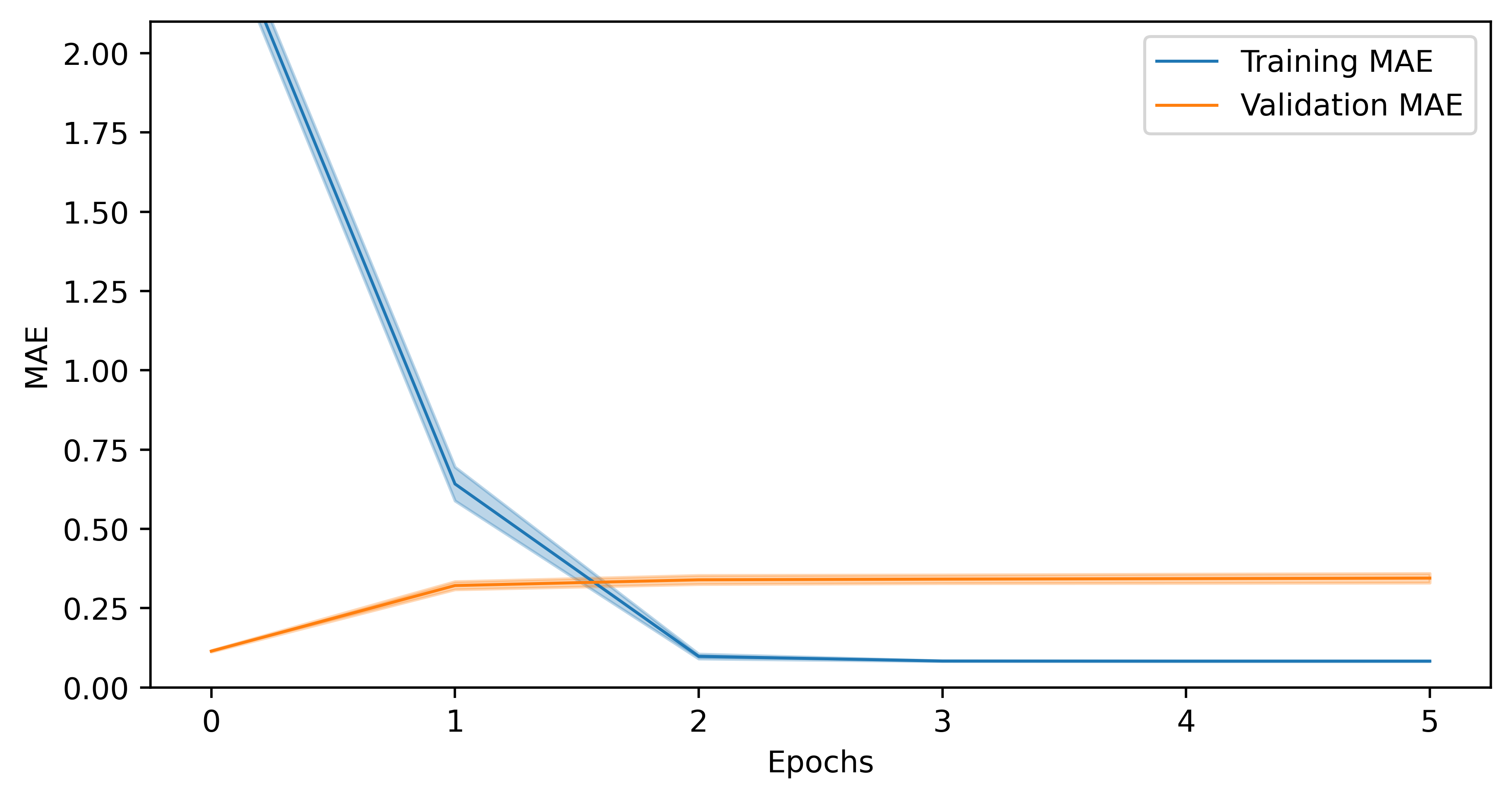}
\caption{Training and validation MAE during the course of training on Task 2, Experiment D.}
\end{figure}

\begin{figure}[!h]
\centering
\includegraphics[width=0.95\linewidth]{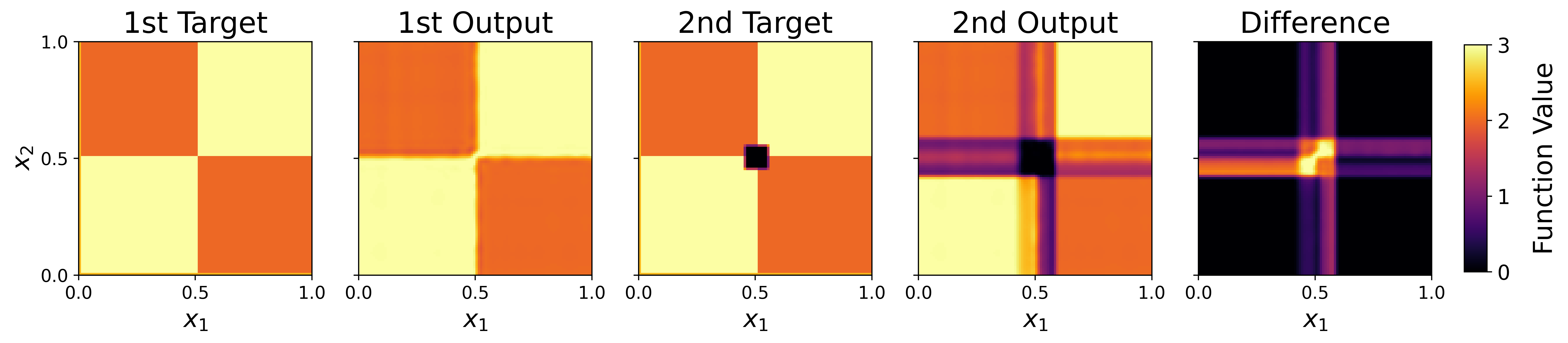}
\caption{Visual inspection of target functions and model outputs over Task 1 and Task 2, Experiment D.}
\end{figure}

\section{Analysis and mathematical proofs}\label{proofs}
\setcounter{theorem}{0}
\setcounter{lemma}{0}
\setcounter{corollary}{0}
\setcounter{definition}{0}
\setcounter{property}{0}

List of all definitions, theorems, corollaries, properties and their proofs are presented for completeness. The numbering of all statements match the main body of the paper. Some of the important results are also presented in the main body of the paper.
\subsection{Stone-Weierstrass Theorem}

Any continuous multi-variable function on a compact space can be uniformly approximated with multi-variable polynomials by the Stone-Weierstrass Theorem. Let $\mathcal{I}$ denote an index set of tuples of natural numbers including zero such that $i_{j} \in \mathbb{N}^{0}$ for all $j \in \mathbb{N}$ with $i = (i_{1},..,i_{n}) \in \mathcal{I}$ and $a_{i} \in \mathbb{R}$. Multi-variable polynomials can be represented as:

$$
y(\vec{\mathbf{x}}) 
= y(x_{1},..,x_{n} ) 
= \sum_{i \in \mathcal{I} } a_{i}  x_{1}^{i_{1}} x_{2}^{i_{2}} ... x_{n}^{i_{n}}
= \sum_{i \in \mathcal{I} } a_{i} \Pi_{j=1}^{n} x_{j}^{i_{j}} 
$$

Each monomial term $a_{i} \Pi_{j=1}^{n} x_{j}^{i_{j}} $ is a product of single-variable functions in each variable.

\subsection{Exponential Representation Theorem}

\begin{lemma}
For any $a_{i} \in \mathbb{R}$, there exists $\gamma_{i}>0$ and $\beta_{i}>0$, such that: 
$a_{i}= \gamma_{i} - \beta_{i}$
\end{lemma}

\begin{proof}
Let $a_{i} \in \mathbb{R}$. Three cases are considered.

If $a_{i} = 0$, then choose $\gamma_{i}=1>0$ and $\beta_{i}=1>0$, such that: $\gamma_{i} - \beta_{i} = 1-1=0=a_{i}$

If $a_{i} > 0$, then choose $\gamma_{i}= a_{i}+1>0$ and $\beta_{i}=1>0$, such that: $\gamma_{i} - \beta_{i} = a_{i}+1-1=a_{i}$

If $a_{i} < 0$, then choose $\gamma_{i}= 1>0$ and $\beta_{i}=1+|a_{i}|>0$, such that: 

$$\gamma_{i} - \beta_{i} = 1-(1+|a_{i}|)=1-1-|a_{i}|=a_{i}$$

\end{proof}


\begin{theorem}[Exponential representation theorem]
\label{thm_exp_rep}
Any multi-variable polynomial function $p(\vec{\mathbf{x}})$ of $n$ variables over the positive orthant, can be exactly represented by continuous single-variable functions $g_{i,j}(x_{j})$ and $h_{i,j}(x_{j})$ in the form:

\begin{equation*}
\begin{split}
p(\vec{\mathbf{x}}) 
= \sum_{i \in \mathcal{I} } \exp( \Sigma_{j=1}^{n} g_{i,j}(x_{j}))    - \exp( \Sigma_{j=1}^{n} h_{i,j}(x_{j}))
\end{split}
\end{equation*}

\end{theorem}

\begin{proof}
Consider any monomial term $a_{i} \Pi_{j=1}^{n} x_{j}^{i_{j}} $ with $a_{i} \in \mathbb{R}$, then by Lemma 1 there exist strictly positive numbers $\gamma_{i}>0$ and $\beta_{i}>0$, such that: 

\begin{equation*} 
\begin{split}
a_{i} \Pi_{j=1}^{n} x_{j}^{i_{j}}  
& = \gamma_{i} \Pi_{j=1}^{n} x_{j}^{i_{j}}  - \beta_{i} \Pi_{j=1}^{n} x_{j}^{i_{j}}   \\
& = \exp(\log(\gamma_{i} \Pi_{j=1}^{n} x_{j}^{i_{j}} )) - \exp(\log(\beta_{i} \Pi_{j=1}^{n} x_{j}^{i_{j}} )) \\
& = \exp(\log(\gamma_{i}) + \Sigma_{j=1}^{n} \log(x_{j}^{i_{j}} )) 
    - \exp(\log(\beta_{i}) +\Sigma_{j=1}^{n} \log(x_{j}^{i_{j}} )) \\
\end{split}
\end{equation*}

The argument of each exponential function is a sum of single-variable functions and constants. Without loss of generality, a set of single-variable functions can be defined such that:

\begin{equation*} 
\begin{split}
a_{i} \Pi_{j=1}^{n} x_{j}^{i_{j}}  
& = \exp( \Sigma_{j=1}^{n} g_{i,j}(x_{j}))
    - \exp( \Sigma_{j=1}^{n} h_{i,j}(x_{j})) 
\end{split}
\end{equation*}

Since this holds for any $a_{i} \Pi_{j=1}^{n} x_{j}^{i_{j}} $ and all $i \in \mathcal{I}$, it follows that:

\begin{equation*} 
\begin{split}
p(\vec{\mathbf{x}}) 
= \sum_{i \in \mathcal{I} } \exp( \Sigma_{j=1}^{n} g_{i,j}(x_{j}))    - \exp( \Sigma_{j=1}^{n} h_{i,j}(x_{j}))
\end{split}
\end{equation*}
\end{proof}

There is a duality between representation and approximation. If any multi-variable polynomial can be exactly represented, then any continuous multi-variable function can be approximated to arbitrary accuracy. 

\subsubsection{Exponential approximation corollary}
\begin{corollary} [Exponential approximation]
For any $\varepsilon>0$, and continuous multi-variable function $y(\vec{\mathbf{x}})$ of $n$ variables over a compact domain in the positive orthant, there exist continuous single-variable functions $g_{i,j}(x_{j})$ and $h_{i,j}(x_{j})$ such that:

\begin{equation*}
\begin{split}
\abs{
y(\vec{\mathbf{x}}) 
- \left(
 \sum_{i \in \mathcal{I} } \exp( \Sigma_{j=1}^{n} g_{i,j}(x_{j}))    - \exp( \Sigma_{j=1}^{n} h_{i,j}(x_{j})) 
\right)
}
< \varepsilon
\end{split}
\end{equation*}

\end{corollary}

\begin{proof}
Fix $\varepsilon>0$, and let $y(\vec{\mathbf{x}})$ be a continuous multi-variable function  of $n$ variables over a compact domain in the positive orthant. 

By the Stone–Weierstrass theorem there exists a multi-variable polynomial $p(\vec{\mathbf{x}})$ over the domain of $y(\vec{\mathbf{x}})$ such that: 

\begin{equation*}
\begin{split}
|y(\vec{\mathbf{x}}) - p(\vec{\mathbf{x}}) | < \varepsilon
\end{split}
\end{equation*}

By \autoref{thm_exp_rep}, for any polynomial $p(\vec{\mathbf{x}})$ over the positive orthant, there exist continuous single-variable functions $g_{i,j}(x_{j})$ and $h_{i,j}(x_{j})$ such that:

\begin{equation*}
\begin{split}
p(\vec{\mathbf{x}}) 
= \sum_{i \in \mathcal{I} } \exp( \Sigma_{j=1}^{n} g_{i,j}(x_{j}))    - \exp( \Sigma_{j=1}^{n} h_{i,j}(x_{j}))
\end{split}
\end{equation*}

It follows that:

\begin{equation*}
\begin{split}
\abs{
y(\vec{\mathbf{x}}) 
- \left(
 \sum_{i \in \mathcal{I} } \exp( \Sigma_{j=1}^{n} g_{i,j}(x_{j}))    - \exp( \Sigma_{j=1}^{n} h_{i,j}(x_{j})) 
\right)
}
< \varepsilon
\end{split}
\end{equation*}

\end{proof}

\subsection{Single-variable function approximators}
Each basis function $S_{i}$ for a uniform cubic B-spline can be obtained by scaling and translating the input of the same activation function. The activation function is denoted $S(x)$ and is given by:

$$ S(x) =\begin{cases} 
      \frac{1}{6} x^{3} &  0 \leq x < 1\\
      \frac{1}{6} \left[-3(x-1)^{3} +3(x-1)^{2} +3(x-1) + 1 \right] &  1 \leq x < 2\\
      \frac{1}{6} \left[3(x-2)^{3} -6(x-2)^{2} + 4 \right]  & 2 \leq x < 3\\
      \frac{1}{6} ( 4-x ) ^{3} &  3 \leq x < 4\\
      0 & otherwise 
   \end{cases}
$$

\begin{figure}[!h]
\centering
\includegraphics[width=0.55\linewidth]{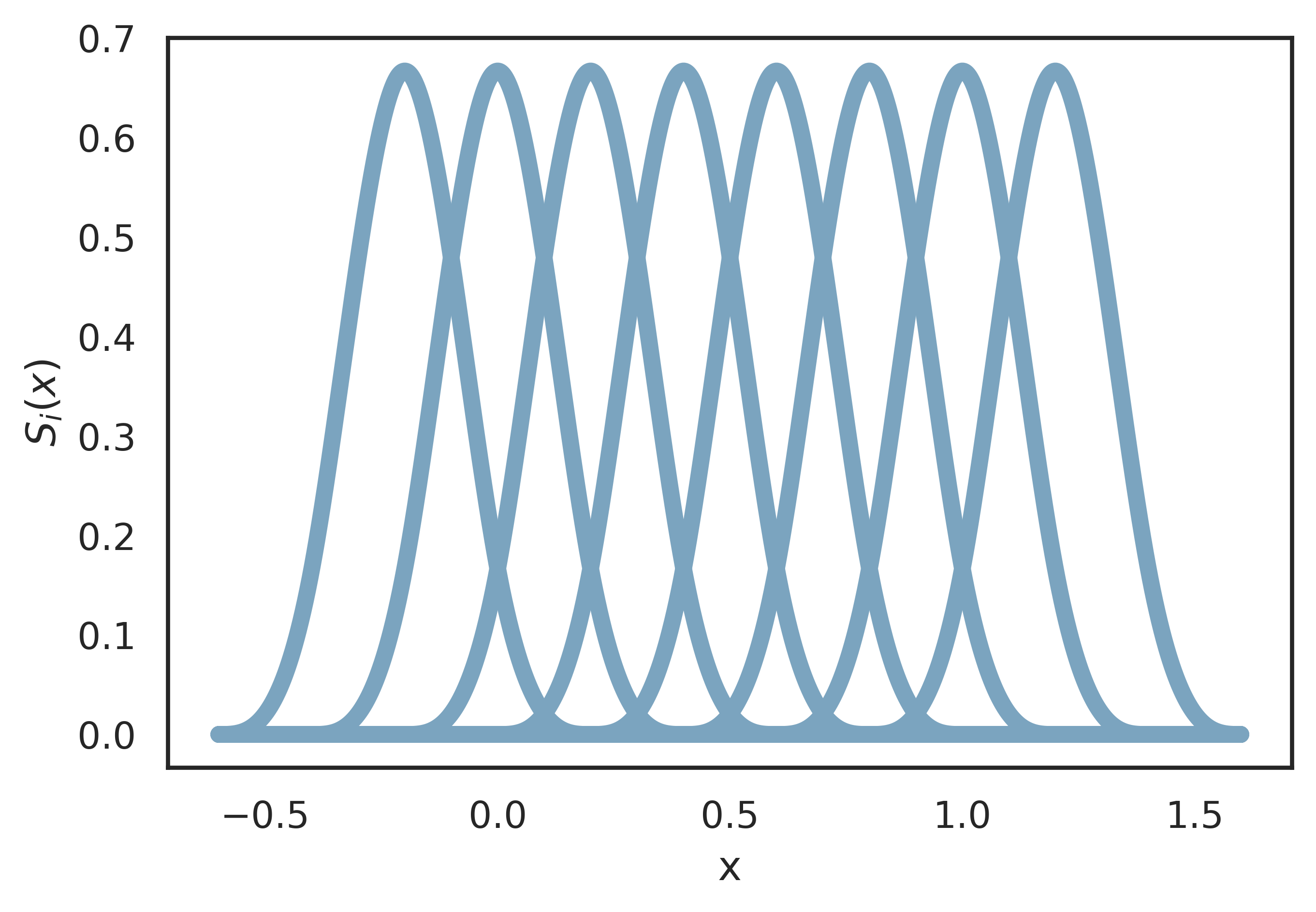}
\caption{Set of eight uniform cubic B-spline basis functions, where $S_{i}(x) = S(w_{i}x + b_{i})$.}
\label{fig:fig_distal_orhogonality}
\end{figure}

\subsubsection{Definition of $\rho$-density B-spline functions}

\begin{definition}[$\rho$-density B-spline function]
\label{rho_density_def}
A $\rho$-density B-spline function is a uniform cubic B-spline function with $2^{\rho+2}$ basis functions:
\begin{equation*} 
\begin{split}
f(x) 
= \sum_{i=1}^{2^{\rho+2}} \theta_{i} S_{i}(x) 
= \sum_{i=1}^{2^{\rho+2}} \theta_{i} S(w_{i}x + b_{i})
= \sum_{i=1}^{2^{\rho+2}} \theta_{i} S( (2^{\rho+2}-3)x + 4 - i )
\end{split}
\end{equation*}
\end{definition}

\subsubsection{Definition of mixed-density B-spline function}
\begin{definition}[mixed-density B-spline function]
\label{mixed_density_def}
A mixed-density B-spline function is a single-variable function approximator that is obtained by summing together different $\rho$-density B-spline functions. Only the maximum $\rho$-density B-spline function has trainable parameters, the others are constant. Mixed-density B-spline functions are of the form:
\begin{equation*}
    \begin{aligned}
f(x) 
&= \sum_{\rho=0}^{r} \sum_{i=1}^{ 2^{\rho+2}} \theta_{\rho,i} S_{\rho,i}(x) 
    \end{aligned}
\end{equation*}
\end{definition}

The maximum density parameters $\theta_{r,i}$ are trainable, but the lower density parameters $\theta_{\rho,i}$ (with $\rho<r$) are in general non-zero constants. The function approximator can be expanded without losing previously learned values. Analytically, we can choose all the new parameters $\theta_{r+1,i} = 0, \; \forall i \in \bf{N}$ such that:

\begin{equation*}
    \begin{aligned}
            f(x) 
&= \sum_{\rho=0}^{r} \sum_{i=1}^{ 2^{\rho+2}} \theta_{\rho,i} S_{\rho,i}(x)
= \sum_{\rho=0}^{r+1} \sum_{i=1}^{ 2^{\rho+2}} \theta_{\rho,i} S_{\rho,i}(x) 
    \end{aligned}
\end{equation*}


\subsection{Atlas architecture}

\subsubsection{Atlas representation theorem}

\begin{theorem}[Atlas representation theorem]
\label{thm_atlas_rep}
Any multi-variable polynomial $p(\vec{\mathbf{x}})$ of $n$ variables over the positive orthant, can be exactly represented by continuous single-variable functions $f_{j}(x_{j})$, $g_{i,j}(x_{j})$, and $h_{i,j}(x_{j})$ in the form:

\begin{equation*} 
\begin{split}
p(\vec{\mathbf{x}}) 
 = \sum_{j=1}^{n} f_{j}(x_{j}) + 
\sum_{k = 1}^{\infty} \frac{1}{k^{2}} \exp( \Sigma_{j=1}^{n} g_{k,j}(x_{j})) 
                - \frac{1}{k^{2}} \exp( \Sigma_{j=1}^{n} h_{k,j}(x_{j})) 
\end{split}
\end{equation*}

\end{theorem}

\begin{proof}
Let $p(\vec{\mathbf{x}})$ be a multi-variable polynomial over the positive orthant: 

$$
p(\vec{\mathbf{x}}) 
= p(x_{1},..,x_{n} ) 
= \sum_{i \in \mathcal{I} } a_{i}  x_{1}^{i_{1}} x_{2}^{i_{2}} ... x_{n}^{i_{n}}
= \sum_{i \in \mathcal{I} } a_{i} \Pi_{j=1}^{n} x_{j}^{i_{j}} 
$$

Consider the set of terms that depend on at most one input variable, or single-variable terms in the expression for the polynomial $p(\vec{\mathbf{x}})$:

$$
\mathcal{P}_{1}
:=\{
a_{i} \Pi_{j=1}^{n} x_{j}^{i_{j}} |\; i \in \mathcal{I}, \; i_{j} \neq 0 \implies i_{k} = 0, \; \forall k \neq j
\}
$$

It is worth noting that $\mathcal{P}_{1}$ contains the constant function. 

Let $\mathcal{Q}$ denote the index set of all single-variable monomial terms:

$$
\mathcal{Q}
:=\{
\; i \; | \; i \in \mathcal{I}, \; i_{j} \neq 0 \implies i_{k} = 0, \; \forall k \neq j
\}
$$

The polynomial $p(\vec{\mathbf{x}})$ can be rewritten in terms of single-variable functions and a residual polynomial function $p_{res}$ as:

\begin{equation*}
\begin{split}
p(\vec{\mathbf{x}}) 
&=   \sum_{i \in \mathcal{Q} } a_{i} \Pi_{j=1}^{n} x_{j}^{i_{j}}  
+   \sum_{i \in \mathcal{I} \setminus \mathcal{Q} } a_{i} \Pi_{j=1}^{n} x_{j}^{i_{j}} \\
&= \sum_{j=1}^{n} f_{j}(x_{j}) + p_{res}(\vec{\mathbf{x}})
\end{split}
\end{equation*}

The single-variable terms can be consumed by a sum of $n$ arbitrary single-variable functions $f_{j}(x_{j})$. 

By \autoref{thm_exp_rep}, for any polynomial $p_{res}(\vec{\mathbf{x}})$ over the positive orthant, there exist continuous single-variable functions $g_{i,j}(x_{j})$ and $h_{i,j}(x_{j})$ such that:

\begin{equation*}
\begin{split}
p_{res}(\vec{\mathbf{x}}) 
= \sum_{i \in \mathcal{I} \setminus \mathcal{Q} } \exp( \Sigma_{j=1}^{n} g_{i,j}(x_{j}))    - \exp( \Sigma_{j=1}^{n} h_{i,j}(x_{j}))
\end{split}
\end{equation*}

Since the index set is countable, one can use another indexing scheme:

\begin{equation*}
\begin{split}
p_{res}(\vec{\mathbf{x}}) 
= \sum_{k = 1}^{\infty}  \exp( \Sigma_{j=1}^{n} g_{k,j}(x_{j})) - \exp( \Sigma_{j=1}^{n} h_{k,j}(x_{j})) 
\end{split}
\end{equation*}

Scale factors can be introduced without changing the representation:

\begin{equation*}
\begin{split}
p_{res}(\vec{\mathbf{x}}) 
&= \sum_{k = 1}^{\infty}  
\exp( \log{k^{2}} -\log{k^{2}} +  \Sigma_{j=1}^{n} g_{k,j}(x_{j})) 
- \exp( \log{k^{2}} -\log{k^{2}} + \Sigma_{j=1}^{n} h_{k,j}(x_{j})) \\
&= \sum_{k = 1}^{\infty}  
 \frac{1}{k^{2}} \exp( \log{k^{2}} +  \Sigma_{j=1}^{n} g_{k,j}(x_{j})) 
- \frac{1}{k^{2}} \exp( \log{k^{2}} + \Sigma_{j=1}^{n} h_{k,j}(x_{j})) \\
\end{split}
\end{equation*}

Since the single-variable functions $g_{i,j}(x_{j})$ and $h_{i,j}(x_{j})$ are arbitrary, one can absorb the constants and redefine $g_{i,j}(x_{j})$ and $h_{i,j}(x_{j})$ to obtain:

\begin{equation*}
\begin{split}
p_{res}(\vec{\mathbf{x}}) 
= \sum_{k = 1}^{\infty} \frac{1}{k^{2}} \exp( \Sigma_{j=1}^{n} g_{k,j}(x_{j})) 
                - \frac{1}{k^{2}} \exp( \Sigma_{j=1}^{n} h_{k,j}(x_{j})) 
\end{split}
\end{equation*}

Substituting the expressions one obtains:

\begin{equation*} 
\begin{split}
p(\vec{\mathbf{x}}) 
 = \sum_{j=1}^{n} f_{j}(x_{j}) + 
\sum_{k = 1}^{\infty} \frac{1}{k^{2}} \exp( \Sigma_{j=1}^{n} g_{k,j}(x_{j})) 
                - \frac{1}{k^{2}} \exp( \Sigma_{j=1}^{n} h_{k,j}(x_{j})) 
\end{split}
\end{equation*}

\end{proof}

There is a duality between representation and approximation. If any multi-variable polynomial can be exactly represented, then any continuous multi-variable function can be approximated to arbitrary accuracy. 

\subsubsection{Atlas definition}

\begin{definition}[Atlas]
\label{atlas_def}
Atlas is a function approximator of $n$ variables, with mixed-density B-spline functions $f_{j}(x_{j})$, $g_{i,j}(x_{j})$, and $h_{i,j}(x_{j})$ in the form:

\begin{equation*} 
\begin{split}
A(\vec{\mathbf{x}}) 
 := \sum_{j=1}^{n} f_{j}(x_{j}) + 
\sum_{k = 1}^{M} \frac{1}{k^{2}} \exp( \Sigma_{j=1}^{n} g_{k,j}(x_{j})) 
                - \frac{1}{k^{2}} \exp( \Sigma_{j=1}^{n} h_{k,j}(x_{j})) 
\end{split}
\end{equation*}

\end{definition}

Atlas is equivalently given by the compact notation:

\begin{equation*} 
\begin{split}
A(\vec{\mathbf{x}}) 
:= & \sum_{j=1}^{n} f_{j}(x_{j}) + 
\sum_{k = 1}^{M} \frac{1}{k^{2}} \exp( \Sigma_{j=1}^{n} g_{k,j}(x_{j}) ) 
                - \frac{1}{k^{2}} \exp( \Sigma_{j=1}^{n} h_{k,j}(x_{j})) \\
= & F(\vec{\mathbf{x}}) + 
\sum_{k = 1}^{M} \frac{1}{k^{2}} \exp(G_{k}(\vec{\mathbf{x}})) 
                - \frac{1}{k^{2}} \exp(H_{k}(\vec{\mathbf{x}})) \\
= & F(\vec{\mathbf{x}}) +  G(\vec{\mathbf{x}})  
                - H(\vec{\mathbf{x}}) \\
\end{split}
\end{equation*}

\subsubsection{Atlas polynomial approximation}
\begin{theorem}[Atlas polynomial approximation]
\label{thm_atlas_polynomial_approximation}
For any multi-variable polynomial $p(\vec{\mathbf{x}})$ over the positive orthant with bounded and compact domain $D(p)$ and $\varepsilon > 0$, there exists an Atlas model $A(\vec{\mathbf{x}})$ such that:

\begin{equation*}
\begin{split}
|p(\vec{\mathbf{x}}) - A(\vec{\mathbf{x}})| < \varepsilon
\end{split}
\end{equation*}

\end{theorem}

\begin{proof}
Let $p(\vec{\mathbf{x}})$ be a multi-variable polynomial $p(\vec{\mathbf{x}})$ of $n$ variables over the positive orthant, and fix $\varepsilon > 0$, and choose: 

$$ \varepsilon = \varepsilon_{1} + \varepsilon_{2}$$

By \autoref{thm_atlas_rep}, there exist continuous single-variable functions $f_{j}(x_{j})$, $g_{i,j}(x_{j})$, and $h_{i,j}(x_{j})$ such that:

\begin{equation*} 
\begin{split}
p(\vec{\mathbf{x}}) 
 = \sum_{j=1}^{n} f_{j}(x_{j}) + 
\sum_{k = 1}^{\infty} \frac{1}{k^{2}} \exp( \Sigma_{j=1}^{n} g_{k,j}(x_{j})) 
                - \frac{1}{k^{2}} \exp( \Sigma_{j=1}^{n} h_{k,j}(x_{j})) 
\end{split}
\end{equation*}

If  $p(\vec{\mathbf{x}})$ has finitely many terms, then let $M$ denote the number of residual terms:

\begin{equation*} 
\begin{split}
p(\vec{\mathbf{x}}) 
= & \sum_{j=1}^{n} f_{j}(x_{j}) + 
\sum_{k = 1}^{M} \frac{1}{k^{2}} \exp( \Sigma_{j=1}^{n} g_{k,j}(x_{j}) ) 
                - \frac{1}{k^{2}} \exp( \Sigma_{j=1}^{n} h_{k,j}(x_{j})) \\
= & F(\vec{\mathbf{x}}) + 
\sum_{k = 1}^{M} \frac{1}{k^{2}} \exp(G_{k}(\vec{\mathbf{x}})) 
                - \frac{1}{k^{2}} \exp(H_{k}(\vec{\mathbf{x}})) \\
= & F(\vec{\mathbf{x}}) +  G(\vec{\mathbf{x}})  
                - H(\vec{\mathbf{x}}) \\
\end{split}
\end{equation*}

Choose mixed-density B-spline functions $f_{j}^{*}(x_{j})$, $g_{i,j}^{*}(x_{j})$, and $h_{i,j}^{*}(x_{j})$ such that the Atlas model $A(\vec{\mathbf{x}})$ is given by:

\begin{equation*} 
\begin{split}
A^{*}(\vec{\mathbf{x}}) 
= & \sum_{j=1}^{n} f_{j}^{*}(x_{j}) + 
\sum_{k = 1}^{M} \frac{1}{k^{2}} \exp( \Sigma_{j=1}^{n} g_{k,j}^{*}(x_{j})) 
                - \frac{1}{k^{2}} \exp( \Sigma_{j=1}^{n} h_{k,j}^{*}(x_{j})) \\
= & F^{*}(\vec{\mathbf{x}}) + 
\sum_{k = 1}^{M} \frac{1}{k^{2}} \exp(G_{k}^{*}(\vec{\mathbf{x}})) 
                - \frac{1}{k^{2}} \exp(H_{k}^{*}(\vec{\mathbf{x}})) \\   
= & F^{*}(\vec{\mathbf{x}}) +  G^{*}(\vec{\mathbf{x}})  
                - H^{*}(\vec{\mathbf{x}}) \\ 
\end{split}
\end{equation*}

Then it follows that,

\begin{equation*}
\begin{split}
\abs{p(\vec{\mathbf{x}}) - A(\vec{\mathbf{x}})} 
= & \abs{ F(\vec{\mathbf{x}}) +  G(\vec{\mathbf{x}})  
                - H(\vec{\mathbf{x}}) 
- \left( F^{*}(\vec{\mathbf{x}}) +  G^{*}(\vec{\mathbf{x}})  
                - H^{*}(\vec{\mathbf{x}}) \right) } \\
= & \abs{ F(\vec{\mathbf{x}}) - F^{*}(\vec{\mathbf{x}}) 
+  G(\vec{\mathbf{x}})  - G^{*}(\vec{\mathbf{x}})
- \left( H(\vec{\mathbf{x}}) - H^{*}(\vec{\mathbf{x}})\right)    
                  } \\
\leq & \abs{ F(\vec{\mathbf{x}}) - F^{*}(\vec{\mathbf{x}})}
+  \abs{G(\vec{\mathbf{x}})  - G^{*}(\vec{\mathbf{x}})}
+  \abs{ H(\vec{\mathbf{x}}) - H^{*}(\vec{\mathbf{x}}) }   \\
\end{split}
\end{equation*}

The first set of functions is easily shown to have bounded error. Choose mixed-density B-spline functions $f_{j}^{*}(x_{j})$ such that:

\begin{equation*}
\begin{split}
| f_{j}(x_{j}) - f_{j}^{*}(x_{j}) | 
 & < \frac{\varepsilon_{1}}{n} \\
\end{split}
\end{equation*}

Then it follows,

\begin{equation*}
\begin{split}
\abs{ F(\vec{\mathbf{x}}) - F^{*}(\vec{\mathbf{x}})} 
= & \abs{ \sum_{j=1}^{n} f_{j}(x_{j}) - \sum_{j=1}^{n} f_{j}^{*}(x_{j}) } \\
\leq & \sum_{j=1}^{n} \abs{ f_{j}(x_{j}) - f_{j}^{*}(x_{j}) } \\
< & 
\varepsilon_{1}   \\
\end{split}
\end{equation*}

The interior functions for the exponential functions are more complicated. 

\begin{remark}
The uniform continuity of exponentials on bouned domains makes it possible to bound the approximation error in each exponential term. The exponential function is uniformly continuous on a compact and bounded subset of the real numbers $\left[a,b \right]$. Thus, for any $\varepsilon_{\exp} > 0$, there exists a $\delta_{\exp} > 0$, such that for every  $x,y \in \left[a,b \right]$:

\begin{equation*}
\begin{split}
 |x-y| 
 & < \delta_{\exp} \implies |\exp(x)-\exp(y)| < \varepsilon_{\exp} \\
\end{split}
\end{equation*}
\end{remark}

For all exponential functions on bounded and compact domains choose:

$$ \varepsilon_{\exp} = \frac{3 \varepsilon_{2}}{\pi^2} $$

Choose the smallest $\delta_{\exp}$ for all $M$ exponential functions, so that the implication holds. Choose $ \delta_{g,k,j} $, such that:

$$ \sum_{j=1}^{n} \delta_{g,k,j} < \delta_{\exp} $$ 

Choose mixed-density B-spline functions $g_{i,j}^{*}(x_{j})$, and $h_{i,j}^{*}(x_{j})$ such that:

\begin{equation*}
\begin{split}
 | g_{k,j}(x_{j}) - g_{k,j}^{*}(x_{j}) | 
 & < \delta_{g,k,j} \\
| h_{k,j}(x_{j}) - h_{k,j}^{*}(x_{j}) | 
 & <  \delta_{g,k,j} \\
\end{split}
\end{equation*}

The interior functions $g_{k,j}^{*}(x_{j})$ have bounded approximation error $\delta_{g,k,j}$ one obtains:


\begin{equation*} 
\begin{split}
    &      \abs{  \sum_{j=1}^{n} g_{k,j}(x_{j}) 
 -      \sum_{j=1}^{n} g_{k,j}^{*}(x_{j}) } \\
\leq&  \sum_{j=1}^{n} \abs{ g_{k,j}(x_{j}) 
 -    g_{k,j}^{*}(x_{j}) } \\
<   &  \sum_{j=1}^{n} \delta_{g,k,j} < \delta_{\exp} \\
\end{split}
\end{equation*}

This implies that:

\begin{equation*} 
\begin{split}
 \abs{ \exp( \Sigma_{j=1}^{n} g_{k,j}(x_{j})) 
 -  \exp( \Sigma_{j=1}^{n} g_{k,j}^{*}(x_{j})) } < & \varepsilon_{\exp}  
 \\
\end{split}
\end{equation*}

Recombining this result with the exponential terms yields:

\begin{equation*} 
\begin{split}
 \abs{G(\vec{\mathbf{x}})  - G^{*}(\vec{\mathbf{x}})} 
= & \abs{ \sum_{k = 1}^{M} \frac{1}{k^{2}} \exp( \Sigma_{j=1}^{n} g_{k,j}(x_{j})) - \sum_{k = 1}^{M} \frac{1}{k^{2}} \exp( \Sigma_{j=1}^{n} g_{k,j}^{*}(x_{j})) }\\
 \abs{G(\vec{\mathbf{x}})  - G^{*}(\vec{\mathbf{x}})} 
\leq &  \sum_{k = 1}^{M} \frac{1}{k^{2}} \abs{ \exp( \Sigma_{j=1}^{n} g_{k,j}(x_{j})) -  \exp( \Sigma_{j=1}^{n} g_{k,j}^{*}(x_{j})) } \\
 \abs{G(\vec{\mathbf{x}})  - G^{*}(\vec{\mathbf{x}})} 
<&  \sum_{k = 1}^{M} \frac{1}{k^{2}} \varepsilon_{\exp} \\
\end{split}
\end{equation*}

The scaling factors of $k^{-2}$ were chosen for convergence, such that:

\begin{equation*} 
\begin{split}
 \abs{G(\vec{\mathbf{x}})  - G^{*}(\vec{\mathbf{x}})} 
<&  \sum_{k = 1}^{\infty} \frac{1}{k^{2}} \varepsilon_{\exp} \\
 \abs{G(\vec{\mathbf{x}})  - G^{*}(\vec{\mathbf{x}})} 
<&  \varepsilon_{\exp} \sum_{k = 1}^{\infty} \frac{1}{k^{2}}  \\
 \abs{G(\vec{\mathbf{x}})  - G^{*}(\vec{\mathbf{x}})} 
<&  \varepsilon_{\exp} \frac{\pi^{2}}{6} = \frac{3 \varepsilon_{2}}{\pi^2} \frac{\pi^{2}}{6} = \frac{\varepsilon_{2}}{2} \\
\end{split}
\end{equation*}

It follows that:

\begin{equation*} 
\begin{split}
 \abs{G(\vec{\mathbf{x}})  - G^{*}(\vec{\mathbf{x}})} 
<& \frac{\varepsilon_{2}}{2}
\end{split}
\end{equation*}

The same argument holds for $\abs{ H(\vec{\mathbf{x}}) - H^{*}(\vec{\mathbf{x}}) }$, and one obtains the result:

\begin{equation*}
\begin{split}
\abs{p(\vec{\mathbf{x}}) - A(\vec{\mathbf{x}})} 
\leq & \abs{ F(\vec{\mathbf{x}}) - F^{*}(\vec{\mathbf{x}})}
+  \abs{G(\vec{\mathbf{x}})  - G^{*}(\vec{\mathbf{x}})}
+  \abs{ H(\vec{\mathbf{x}}) - H^{*}(\vec{\mathbf{x}}) }   \\
\abs{p(\vec{\mathbf{x}}) - A(\vec{\mathbf{x}})} 
< & \varepsilon_{1}
+  \frac{\varepsilon_{2}}{2}
+  \frac{\varepsilon_{2}}{2}   \\
\abs{p(\vec{\mathbf{x}}) - A(\vec{\mathbf{x}})} 
< & \varepsilon_{1}
+  \varepsilon_{2}   \\
\end{split}
\end{equation*}

Finally, 

\begin{equation*}
\begin{split}
\abs{p(\vec{\mathbf{x}}) - A(\vec{\mathbf{x}})} 
< & \varepsilon   \\
\end{split}
\end{equation*}

\end{proof}

\subsubsection{Universal function approximation theorem}

\begin{theorem}[Atlas universal function approximation]

For any $\varepsilon>0$, and continuous multi-variable function $y(\vec{\mathbf{x}})$ of $n$ variables over a compact domain in the positive orthant, there exists an Atlas model $A(\vec{\mathbf{x}})$ such thtat:

\begin{equation*}
\begin{split}
| y(\vec{\mathbf{x}}) - A(\vec{\mathbf{x}}) | < \varepsilon
\end{split}
\end{equation*}

\end{theorem}

\begin{proof}
Let $\varepsilon>0$, and $y(\vec{\mathbf{x}})$ be a continuous multi-variable function  of $n$ variables over a compact domain in the positive orthant. Choose $\varepsilon_{1}+\varepsilon_{2} = \varepsilon$.

By the Stone–Weierstrass theorem there exists a multi-variable polynomial $p(\vec{\mathbf{x}})$ over the domain of $y(\vec{\mathbf{x}})$ such that: 

\begin{equation*}
\begin{split}
|y(\vec{\mathbf{x}}) - p(\vec{\mathbf{x}}) | < \varepsilon_{1}
\end{split}
\end{equation*}

By \autoref{thm_atlas_polynomial_approximation} an Atlas model can approximate the polynomial $p(\vec{\mathbf{x}})$ to arbitrary precision:

\begin{equation*}
\begin{split}
|p(\vec{\mathbf{x}}) - A(\vec{\mathbf{x}})| < \varepsilon_{2}
\end{split}
\end{equation*}

It follows from the triangle inequality that:

\begin{equation*}
\begin{split}
| y(\vec{\mathbf{x}}) - A(\vec{\mathbf{x}}) | 
 & \leq 
 | y(\vec{\mathbf{x}}) - p(\vec{\mathbf{x}}) | + | p(\vec{\mathbf{x}}) - A(\vec{\mathbf{x}}) | \\
| y(\vec{\mathbf{x}}) - A(\vec{\mathbf{x}}) |
 & < \varepsilon_{1}+\varepsilon_{2} \\
\end{split}
\end{equation*}

Finally,

\begin{equation*}
\begin{split}
 | y(\vec{\mathbf{x}}) - A(\vec{\mathbf{x}}) |
& < \varepsilon
\end{split}
\end{equation*}

\end{proof}

\subsection{Atlas properties}

\subsubsection{Atlas expansion}

The number of exponential terms can be increased without changing the output of the model. We can choose to initialise $G_{M+1}(\vec{\mathbf{x}}) = 0$ and $H_{M+1}(\vec{\mathbf{x}}) = 0$, such that the model capacity can be increased without changing the output of the model:

\begin{equation*}
    \begin{aligned}
A_{M+1}(\vec{\mathbf{x}}) 
&= \sum_{k = 1}^{M+1} \frac{1}{k^{2}} \exp(G_{k}(\vec{\mathbf{x}})) 
                - \frac{1}{k^{2}} \exp(H_{k}(\vec{\mathbf{x}})) \\
&   = \frac{1}{(M+1)^{2}} \exp(G_{M+1}(\vec{\mathbf{x}})) 
    - \frac{1}{(M+1)^{2}} \exp(H_{M+1}(\vec{\mathbf{x}}))
    + A(\vec{\mathbf{x}})  \\
&   = \frac{1}{(M+1)^{2}} \exp(0) 
    - \frac{1}{(M+1)^{2}} \exp(0)
    + A(\vec{\mathbf{x}})  \\
&   =  A(\vec{\mathbf{x}})  
    \end{aligned}
\end{equation*}

The density of basis functions in Atlas can also be incremented without changing the learned output of the model. The density of basis functions can also be increased without changing the output for any of mixed-density B-spline functions $\Psi$ of the form:
\begin{equation*}
    \begin{aligned}
\Psi(x) 
&= \sum_{\rho=0}^{r} \sum_{i=1}^{ 2^{\rho+2}} \theta_{\rho,i} S_{\rho,i}(x) 
    \end{aligned}
\end{equation*}

Analytically, we can choose all the new parameters $\theta_{r+1,i} = 0, \; \forall i \in \bf{N}$ such that:

\begin{equation*}
    \begin{aligned}
            \Psi(x) 
&= \sum_{\rho=0}^{r} \sum_{i=1}^{ 2^{\rho+2}} \theta_{\rho,i} S_{\rho,i}(x)
= \sum_{\rho=0}^{r+1} \sum_{i=1}^{ 2^{\rho+2}} \theta_{\rho,i} S_{\rho,i}(x) 
    \end{aligned}
\end{equation*}

The last thing to note is that only the parameters for the largest specified density are trainable, in contrast to smaller density parameters that are fixed constants.


\subsubsection{Atlas sparsity}
\begin{property}[Sparsity]
For any $\vec{\mathbf{x}} \in D(A) \subset R^{n}$ and bounded trainable parameters $\theta_{i}$ with index set $\Theta$, the gradient vector of trainable parameters for Atlas is sparse:

$$ 
\norm{ \grad_{\vec{\mathbf{\theta}}} A(\vec{\mathbf{x}} )}_{0} 
= \sum_{i \in \Theta}  d_{Hamming} \left(\frac{\partial A}{\partial \theta_{i}} (\vec{\mathbf{x}}),0  \right)
\leq 4 n (2M+1)
$$
\end{property} 

\begin{proof} Let $A( \vec{\mathbf{x}} )$ denote some Atlas model, with mixed-density B-spline functions $f_{j}(x_{j})$, $g_{i,j}(x_{j})$, and $h_{i,j}(x_{j})$ in the form:

\begin{equation*} 
\begin{split}
A(\vec{\mathbf{x}}) 
 = \sum_{j=1}^{n} f_{j}(x_{j}) + 
\sum_{k = 1}^{M} \frac{1}{k^{2}} \exp( \Sigma_{j=1}^{n} g_{k,j}(x_{j})) 
                - \frac{1}{k^{2}} \exp( \Sigma_{j=1}^{n} h_{k,j}(x_{j})) 
\end{split}
\end{equation*}

Each mixed-density B-spline function has its own parameters that are independent of every other mixed-density B-spline. The mixed-density B-splines function $\Psi(x)$ is by definition given by:

\begin{equation*}
    \begin{aligned}
\Psi(x) 
&= \sum_{\rho=0}^{r} \sum_{i=1}^{ 2^{\rho+2}} \theta_{\rho,i} S_{\rho,i}(x) 
    \end{aligned}
\end{equation*}

Thus for every mixed-density B-spline function in $A( \vec{\mathbf{x}} )$:

\begin{equation*}
    \begin{aligned}
f_{j}(x_{j}) 
&= \sum_{\rho=0}^{r} \sum_{i=1}^{ 2^{\rho+2}} \theta_{f,(\rho,i,j)} S_{\rho,i}(x_{j}) \\
g_{k,j}(x_{j})) 
&= \sum_{\rho=0}^{r} \sum_{i=1}^{ 2^{\rho+2}} \theta_{g,(\rho,i,k,j)} S_{\rho,i}(x) \\
h_{k,j}(x_{j})) 
&= \sum_{\rho=0}^{r} \sum_{i=1}^{ 2^{\rho+2}} \theta_{h,(\rho,i,k,j)} S_{\rho,i}(x) \\
    \end{aligned}
\end{equation*}

Only the maximum density basis function $\rho=r$ have trainable parameters. The maximum density $r$-density B-spline functions are uniform B-spline functions with trainable parameters. There are at most four basis functions that are non-zero for any given $x_{j}$, and as such the gradient vector with respect to trainable parameters will have at most four non-zero entries for each $r$-density B-spline function, the same four parameters for each mixed-density B-spline functions $f_{j}(x_{j})$, $g_{i,j}(x_{j})$, and $h_{i,j}(x_{j})$. One simply needs to count the number of mixed-density B-spline functions. 

The number of mixed-density B-spline functions labeled $f_{j}(x_{j})$ are in total $n$, with $4$ active trainable parameters each. 

The number of functions labeled $g_{k,j}(x_{j})$ are in total $nM$. For each $M$, there are $n$ mixed-density B-spline functions, with $4$ active trainable parameters each. 

The number of mixed-density B-spline functions labeled $h_{k,j}(x_{j})$ are in total $nM$. For each $M$, there are $n$ mixed-density B-spline functions, with $4$ active trainable parameters each. 

The total number of active trainable parameters is thus:

\begin{equation*}
    \begin{aligned}
4n + 4nM + 4nM 
&= 4 n (2M+1) \\
    \end{aligned}
\end{equation*}

\end{proof}

\begin{remark} The total number of trainable parameters for each mixed-density B-spline function is $2^{r+2}$. For a fixed number of variables $n$, the model has a total of $2^{r+2}n(2M+1)$ trainable parameters. The gradient vector has a maximum of $4 n (2M+1)$ non-zero entries, which is independent of $r$. Recall that only the maximum density ($\rho=r$) cubic B-spline function has trainable parameters. The fraction of trainable basis functions that are active is at most $2^{-r}$. Sparsity entails efficient implementation, and suggests possible memory retention and robustness to catastrophic forgetting.
\end{remark}

It is worth noting that the total number of parameters (including constants) is:

\begin{equation*}
    \begin{aligned}
\text{Total number of parameters} \propto \sum_{\rho = 0}^{r} 2^{\rho+2}n(2M+1) \approx 2^{r+1}n(2M+1) \\
    \end{aligned}
\end{equation*}

\subsubsection{Atlas gradient flow attenuation}
\begin{property}[Gradient flow attenuation]
For any $\vec{\mathbf{x}} \in D(A) \subset R^{n}$ and bounded trainable parameters $\theta_{i}$ with index set $\Theta$: if all the mixed-density B-spline functions are bounded, then the gradient vector of trainable parameters for Atlas is bounded:

$$ 
\norm{ \grad_{\vec{\mathbf{\theta}}} A( \vec{\mathbf{x}} )}_{1} 
= \sum_{i \in \Theta}  \left| \frac{\partial A}{\partial \theta_{i}} ( \vec{\mathbf{x}} ) \right| 
< U
$$
\end{property}

\begin{proof} Let $A( \vec{\mathbf{x}} )$ denote some Atlas model, with mixed-density B-spline functions $f_{j}(x_{j})$, $g_{i,j}(x_{j})$, and $h_{i,j}(x_{j})$ in the form:

\begin{equation*} 
\begin{split}
A(\vec{\mathbf{x}}) 
= & \sum_{j=1}^{n} f_{j}(x_{j}) + 
\sum_{k = 1}^{M} \frac{1}{k^{2}} \exp( \Sigma_{j=1}^{n} g_{k,j}(x_{j})) 
                - \frac{1}{k^{2}} \exp( \Sigma_{j=1}^{n} h_{k,j}(x_{j})) \\
= & F(\vec{\mathbf{x}}) + 
\sum_{k = 1}^{M} \frac{1}{k^{2}} \exp(G_{k}(\vec{\mathbf{x}})) 
                - \frac{1}{k^{2}} \exp(H_{k}(\vec{\mathbf{x}})) \\   
= & F(\vec{\mathbf{x}}) +  G(\vec{\mathbf{x}})  
                - H(\vec{\mathbf{x}}) \\ 
\end{split}
\end{equation*}

With each mixed-density B-spline function in $A( \vec{\mathbf{x}} )$ given by:

\begin{equation*}
    \begin{aligned}
f_{j}(x_{j}) 
&= \sum_{\rho=0}^{r} \sum_{i=1}^{ 2^{\rho+2}} \theta_{f,(\rho,i,j)} S_{\rho,i}(x_{j}) \\
g_{k,j}(x_{j})) 
&= \sum_{\rho=0}^{r} \sum_{i=1}^{ 2^{\rho+2}} \theta_{g,(\rho,i,k,j)} S_{\rho,i}(x) \\
h_{k,j}(x_{j})) 
&= \sum_{\rho=0}^{r} \sum_{i=1}^{ 2^{\rho+2}} \theta_{h,(\rho,i,k,j)} S_{\rho,i}(x) \\
    \end{aligned}
\end{equation*}

The norm of the gradient of $A( \vec{\mathbf{x}} )$ with respect to trainable parameters is given by:

\begin{equation*} 
\begin{split}
\norm{ \grad_{\vec{\mathbf{\theta}}} A( \vec{\mathbf{x}} )}_{1}   
= & 
\norm{ 
\grad_{\vec{\mathbf{\theta}}} 
\left(
    F(\vec{\mathbf{x}}) 
+   G(\vec{\mathbf{x}})  
-   H(\vec{\mathbf{x}}) 
\right) }_{1}  \\
\norm{ \grad_{\vec{\mathbf{\theta}}} A( \vec{\mathbf{x}} )}_{1}  
= & 
\norm{ 
    \grad_{\vec{\mathbf{\theta}}} F(\vec{\mathbf{x}}) 
+   \grad_{\vec{\mathbf{\theta}}} G(\vec{\mathbf{x}})  
-   \grad_{\vec{\mathbf{\theta}}} H(\vec{\mathbf{x}}) 
 }_{1}  \\
\norm{ \grad_{\vec{\mathbf{\theta}}} A( \vec{\mathbf{x}} )}_{1}  
\leq & 
\norm{\grad_{\vec{\mathbf{\theta}}} F(\vec{\mathbf{x}})}_{1}
+
\norm{\grad_{\vec{\mathbf{\theta}}} G(\vec{\mathbf{x}})}_{1}
+
\norm{\grad_{\vec{\mathbf{\theta}}} H(\vec{\mathbf{x}})}_{1}  \\
\end{split}
\end{equation*}

The first term is bounded,

\begin{equation*} 
\begin{split}
\norm{\grad_{\vec{\mathbf{\theta}}} F(\vec{\mathbf{x}})}_{1} 
= & 
\norm{\grad_{\vec{\mathbf{\theta}}} 
\left( 
\sum_{j=1}^{n} f_{j}(x_{j}) 
\right) 
}_{1} 
\\
= & 
\norm{ 
\sum_{j=1}^{n} \grad_{\vec{\mathbf{\theta}}} f_{j}(x_{j}) 
}_{1} 
\\
\leq & 
\sum_{j=1}^{n} 
\norm{
\grad_{\vec{\mathbf{\theta}}} f_{j}(x_{j}) 
}_{1} 
\\
\end{split}
\end{equation*}

Substituting the expression for each $f_{j}(x_{j})$, all lower densities have constant parameters:

\begin{equation*} 
\begin{split}
\norm{\grad_{\vec{\mathbf{\theta}}} F(\vec{\mathbf{x}})}_{1} 
\leq & 
\sum_{j=1}^{n} 
\norm{
\grad_{\vec{\mathbf{\theta}}} 
\left(
\sum_{\rho=0}^{r} \sum_{i=1}^{ 2^{\rho+2}} \theta_{f,(\rho,i,j)} S_{\rho,i}(x_{j})
\right)
}_{1} 
\\
\norm{\grad_{\vec{\mathbf{\theta}}} F(\vec{\mathbf{x}})}_{1} 
\leq & 
\sum_{j=1}^{n} 
\norm{
\grad_{\vec{\mathbf{\theta}}} 
\left(
\sum_{i=1}^{ 2^{r+2}} \theta_{f,(r,i,j)} S_{r,i}(x_{j})
\right)
}_{1} 
\\
\norm{\grad_{\vec{\mathbf{\theta}}} F(\vec{\mathbf{x}})}_{1} 
\leq & 
\sum_{j=1}^{n} 
\sum_{i=1}^{ 2^{r+2}}
\norm{
\grad_{\vec{\mathbf{\theta}}} 
\left(
 \theta_{f,(r,i,j)} S_{r,i}(x_{j})
\right)
}_{1} 
\\
\norm{\grad_{\vec{\mathbf{\theta}}} F(\vec{\mathbf{x}})}_{1} 
\leq & 
\sum_{j=1}^{n} 
\sum_{i=1}^{ 2^{r+2}}
\abs{
S_{r,i}(x_{j})
}
\\
\end{split}
\end{equation*}

Each basis function is continuous and bounded by some positive constant $u>0$, such that $S(x)<u$ regardless of its density, and it follows that:

\begin{equation*} 
\begin{split}
\norm{\grad_{\vec{\mathbf{\theta}}} F(\vec{\mathbf{x}})}_{1} 
\leq & 
\sum_{j=1}^{n} 
\sum_{i=1}^{ 2^{r+2}} u
\\
\end{split}
\end{equation*}

The last thing to include is that at most four basis functions are non-zero, regardless of the value of $r$, so a tighter upper bound is:

\begin{equation*} 
\begin{split}
\norm{\grad_{\vec{\mathbf{\theta}}} F(\vec{\mathbf{x}})}_{1} 
\leq & 
\sum_{j=1}^{n} 
\sum_{i=1}^{4} u = 4nu
\\
\end{split}
\end{equation*}

\begin{remark}
Each $\rho$-density B-spline function has at most four active basis functions, and each mixed-density B-spline function has $r+1$ different $\rho$-density B-spline functions. If the lower densities $\rho<r$ were also trainable, then this upper bound would instead be $4nu(r+1)$. This is why only the maximum density was chosen to be trainable.
\end{remark}

The exponential terms are more complicated. 

\begin{equation*} 
\begin{split}
\norm{\grad_{\vec{\mathbf{\theta}}} G(\vec{\mathbf{x}})}_{1} 
= & 
\norm{\grad_{\vec{\mathbf{\theta}}} 
\left( 
\sum_{k = 1}^{M} \frac{1}{k^{2}} \exp(G_{k}(\vec{\mathbf{x}})) 
\right) 
}_{1} 
\\
\norm{\grad_{\vec{\mathbf{\theta}}} G(\vec{\mathbf{x}})}_{1} 
\leq & 
\sum_{k = 1}^{M} \frac{1}{k^{2}}
\norm{\grad_{\vec{\mathbf{\theta}}} 
\left( 
 \exp(G_{k}(\vec{\mathbf{x}})) 
\right) 
}_{1} 
\\
\norm{\grad_{\vec{\mathbf{\theta}}} G(\vec{\mathbf{x}})}_{1} 
\leq & 
\sum_{k = 1}^{M} \frac{1}{k^{2}}
\norm{
 \exp(G_{k}(\vec{\mathbf{x}}))
 \grad_{\vec{\mathbf{\theta}}} 
\left( 
G_{k}(\vec{\mathbf{x}}) 
\right) 
}_{1} 
\\
\norm{\grad_{\vec{\mathbf{\theta}}} G(\vec{\mathbf{x}})}_{1} 
\leq & 
\sum_{k = 1}^{M} \frac{1}{k^{2}}
\exp(G_{k}(\vec{\mathbf{x}}))
\norm{
 \grad_{\vec{\mathbf{\theta}}} 
\left( 
G_{k}(\vec{\mathbf{x}}) 
\right) 
}_{1} 
\\
\end{split}
\end{equation*}

Each mixed-density B-spline function is bounded, so

\begin{equation*}
    \begin{aligned}
\abs{g_{k,j}(x_{j})) }
&= \abs{\sum_{\rho=0}^{r} \sum_{i=1}^{ 2^{\rho+2}} \theta_{g,(\rho,i,k,j)} S_{\rho,i}(x)} < u_{g,(k,j)}\\
    \end{aligned}
\end{equation*}

Since $n$ is fixed and finite, the functions $G_{k}(\vec{\mathbf{x}})$ are bounded:

\begin{equation*}
    \begin{aligned}
\abs{G_{k}(\vec{\mathbf{x}})} 
&=  \abs{\sum_{j=1}^{n} g_{k,j}(x_{j})} 
\leq \sum_{j=1}^{n} \abs{g_{k,j}(x_{j})}
< \sum_{j=1}^{n} u_{g,(k,j)} = u_{g,(k)}
\\
    \end{aligned}
\end{equation*}

Since this is true for each $G_{k}$, one can choose the maximum bound:

\begin{equation*}
    \begin{aligned}
u_{g}
&= \max_{k=1, \dots , M }
\{ u_{g,(k)} \} 
\\
    \end{aligned}
\end{equation*}

It is evident that:

\begin{equation*}
G_{k}(\vec{\mathbf{x}})  
\leq \abs{G_{k}(\vec{\mathbf{x}})} 
< u_{g}
\end{equation*}

Since the exponential function is monotonic increasing:

\begin{equation*}
\exp(G_{k}(\vec{\mathbf{x}}) ) 
\leq \exp(\abs{G_{k}(\vec{\mathbf{x}})} )
< \exp(u_{g})
\end{equation*}

This result can be substituted back,

\begin{equation*} 
\begin{split}
\norm{\grad_{\vec{\mathbf{\theta}}} G(\vec{\mathbf{x}})}_{1} 
< & 
\sum_{k = 1}^{M} \frac{1}{k^{2}}
\exp(u_{g})
\norm{
 \grad_{\vec{\mathbf{\theta}}} 
\left( 
G_{k}(\vec{\mathbf{x}}) 
\right) 
}_{1} 
\\
\end{split}
\end{equation*}

It should be noted that $G_{k}(\vec{\mathbf{x}})$ and $F(\vec{\mathbf{x}})$ have similar structure such that:

\begin{equation*} 
\norm{
 \grad_{\vec{\mathbf{\theta}}} 
\left( 
G_{k}(\vec{\mathbf{x}}) 
\right) 
}_{1}  
= \norm{
 \grad_{\vec{\mathbf{\theta}}} 
\left( 
F(\vec{\mathbf{x}}) 
\right) 
}_{1}  
\end{equation*}

This is true, even though $G_{k}(\vec{\mathbf{x}}) \neq F(\vec{\mathbf{x}})$, because the same set of basis functions are used, with different coefficient parameters being the only difference. The consequence is that:

\begin{equation*} 
\begin{split}
\norm{\grad_{\vec{\mathbf{\theta}}} G(\vec{\mathbf{x}})}_{1} 
< & 
\sum_{k = 1}^{M} \frac{1}{k^{2}}
\exp(u_{g})
\norm{
 \grad_{\vec{\mathbf{\theta}}} 
\left( 
F(\vec{\mathbf{x}}) 
\right) 
}_{1} 
\\
\end{split}
\end{equation*}

Substituting previously shown results gives:

\begin{equation*} 
\begin{split}
\norm{\grad_{\vec{\mathbf{\theta}}} G(\vec{\mathbf{x}})}_{1} 
< & 
\sum_{k = 1}^{M} \frac{1}{k^{2}}
\exp(u_{g})
4nu
\\
\norm{\grad_{\vec{\mathbf{\theta}}} G(\vec{\mathbf{x}})}_{1} 
< & 
4nu \exp(u_{g}) 
\sum_{k = 1}^{\infty} \frac{1}{k^{2}}
\\
\norm{\grad_{\vec{\mathbf{\theta}}} G(\vec{\mathbf{x}})}_{1} 
< & 
4nu \exp(u_{g}) 
\frac{\pi^{2}}{6}
< 4nu \pi^{2} \exp(u_{g}) 
\\
\end{split}
\end{equation*}

The same argument can be used to find an upper bound for $\norm{\grad_{\vec{\mathbf{\theta}}} H(\vec{\mathbf{x}})}_{1}$

\begin{equation*} 
\begin{split}
\norm{\grad_{\vec{\mathbf{\theta}}} H(\vec{\mathbf{x}})}_{1} 
< & 
4nu \exp(u_{h}) 
\frac{\pi^{2}}{6}
< 4nu \pi^{2} \exp(u_{h}) 
\\
\end{split}
\end{equation*}


The original expression of interest was:

\begin{equation*} 
\begin{split}
\norm{ \grad_{\vec{\mathbf{\theta}}} A( \vec{\mathbf{x}} )}_{1}  
\leq & 
\norm{\grad_{\vec{\mathbf{\theta}}} F(\vec{\mathbf{x}})}_{1}
+
\norm{\grad_{\vec{\mathbf{\theta}}} G(\vec{\mathbf{x}})}_{1}
+
\norm{\grad_{\vec{\mathbf{\theta}}} H(\vec{\mathbf{x}})}_{1}  \\
\norm{ \grad_{\vec{\mathbf{\theta}}} A( \vec{\mathbf{x}} )}_{1}  
< & 
4 n u
+
4nu \pi^{2} \exp(u_{g}) 
+
4nu \pi^{2} \exp(u_{h})   \\
\norm{ \grad_{\vec{\mathbf{\theta}}} A( \vec{\mathbf{x}} )}_{1}  
< & 
4 n u \pi^{2}
+
4nu \pi^{2} \exp(u_{g}) 
+
4nu \pi^{2} \exp(u_{h})   \\
\norm{ \grad_{\vec{\mathbf{\theta}}} A( \vec{\mathbf{x}} )}_{1}  
< & 
4 n u \pi^{2} 
\left( 1 +\exp(u_{g}) + \exp(u_{h})  \right) 
\\
\end{split}
\end{equation*}

Let the upper bound $U$ be given by:

\begin{equation*} 
U = 4 n u \pi^{2} \left( 1 +\exp(u_{g}) + \exp(u_{h})  \right)
\end{equation*}

From the definition of the norm $\norm{}_{1}$, and the trainable parameters $\theta_{i}$ with index set $\Theta$ one has that:

$$ 
\norm{ \grad_{\vec{\mathbf{\theta}}} A( \vec{\mathbf{x}} )}_{1} 
= \sum_{i \in \Theta}  \left| \frac{\partial A}{\partial \theta_{i}} ( \vec{\mathbf{x}} ) \right| 
$$

Finally, 

$$ 
\norm{ \grad_{\vec{\mathbf{\theta}}} A( \vec{\mathbf{x}} )}_{1} 
= \sum_{i \in \Theta}  \left| \frac{\partial A}{\partial \theta_{i}} ( \vec{\mathbf{x}} ) \right| 
< U
$$

\end{proof}

\begin{remark}
For a fixed number of variables $n$, the model has a total of $n2^{r+2}(2M+1)$ trainable parameters. The factor of $k^{-2}$ inside the expression for Atlas is necessary to ensure the sum is convergent in the limit of infinitely many exponential terms $M \to \infty$. Only the maximum density ($\rho=r$) cubic B-spline function has trainable parameters, so that the gradient vector is bounded in the limit of arbitrarily large densities $r \to \infty$. It is worth recalling that at most four basis functions are active for uniform cubic B-spline functions, regardless of the density, but the smaller densities cannot be trainable, otherwise this property does not hold. The gradient vector has bounded norm for any number of basis functions and exponential terms. The bounded gradient vector implies that Atlas is numerically stable during training, regardless of its size or parameter count.
\end{remark}

\subsubsection{Atlas distal orthogonality}

\begin{property}[Distal orthogonality]
For any Atlas model $A(\vec{\mathbf{x}})$ and $\forall \; \vec{\mathbf{x}},\vec{\mathbf{y}} \in D(A) \subset R^{n}$ and trainable parameters $\theta_{i}$, there exists a $\delta>0$ such that:

$$
 \min_{j=1, \dots , n }
\{ |x_{j} - y_{j}| \} > \delta 
\implies  
\langle
\grad_{\vec{\mathbf{\theta}}} A(\vec{\mathbf{x}}) 
, 
\grad_{\vec{\mathbf{\theta}}} A(\vec{\mathbf{y}})
\rangle
= 0
$$
\end{property}

\begin{proof} Let $A( \vec{\mathbf{x}} )$ denote some Atlas model, with mixed-density B-spline functions $f_{j}(x_{j})$, $g_{i,j}(x_{j})$, and $h_{i,j}(x_{j})$ in the form:

\begin{equation*} 
\begin{split}
A(\vec{\mathbf{x}}) 
= & \sum_{j=1}^{n} f_{j}(x_{j}) + 
\sum_{k = 1}^{M} \frac{1}{k^{2}} \exp( \Sigma_{j=1}^{n} g_{k,j}(x_{j})) 
                - \frac{1}{k^{2}} \exp( \Sigma_{j=1}^{n} h_{k,j}(x_{j})) \\
= & F(\vec{\mathbf{x}}) + 
\sum_{k = 1}^{M} \frac{1}{k^{2}} \exp(G_{k}(\vec{\mathbf{x}})) 
                - \frac{1}{k^{2}} \exp(H_{k}(\vec{\mathbf{x}})) \\   
= & F(\vec{\mathbf{x}}) +  G(\vec{\mathbf{x}})  
                - H(\vec{\mathbf{x}}) \\ 
\end{split}
\end{equation*}

With each mixed-density B-spline function in $A( \vec{\mathbf{x}} )$ given by:

\begin{equation*}
    \begin{aligned}
f_{j}(x_{j}) 
&= \sum_{\rho=0}^{r} \sum_{i=1}^{ 2^{\rho+2}} \theta_{f,(\rho,i,j)} S_{\rho,i}(x_{j}) \\
g_{k,j}(x_{j})) 
&= \sum_{\rho=0}^{r} \sum_{i=1}^{ 2^{\rho+2}} \theta_{g,(\rho,i,k,j)} S_{\rho,i}(x_{j}) \\
h_{k,j}(x_{j})) 
&= \sum_{\rho=0}^{r} \sum_{i=1}^{ 2^{\rho+2}} \theta_{h,(\rho,i,k,j)} S_{\rho,i}(x_{j}) \\
    \end{aligned}
\end{equation*}

Any mixed-density functions $\Phi$ and $\Psi$ that act on different components of the input must have orthogonal parameter gradients, since each input variable has its own associated parameters:

$$
\langle
\grad_{\vec{\mathbf{\theta}}} \Phi(x_{i}) 
, 
\grad_{\vec{\mathbf{\theta}}} \Psi(y_{j})
\rangle
= 0
\; \forall \; i \neq j
$$

Generally, since all mixed-density functions have parameters that are independent of each other it follows that for any mixed-density B-splines $\Phi$ and $\Psi$:

$$
\langle
\grad_{\vec{\mathbf{\theta}}} \Phi(x_{j}) 
, 
\grad_{\vec{\mathbf{\theta}}} \Psi(y_{j})
\rangle
= 0
\; \forall \; \Phi \neq \Psi
$$

Thus, one need only compare the parameter gradients of each mixed-density B-spline function $\Psi$ with itself. The inner-product of the parameter gradient of $\Psi$ evaluated on two different inputs is given by:

$$
\langle
\grad_{\vec{\mathbf{\theta}}} \Psi(x_{j}) 
, 
\grad_{\vec{\mathbf{\theta}}} \Psi(y_{j})
\rangle
$$

The inner-product given above is not zero in general. However, as illustrated in Figure~\ref{fig:fig_distal_orhogonality}, for any mixed-density B-spline function $\Psi$ there exist a $\delta>0$, such that:

$$
|x_{j} - y_{j}| > \delta 
\implies  \langle \grad_{\vec{\mathbf{\theta}}} \Psi(x_{j}),\grad_{\vec{\mathbf{\theta}}} \Psi(y_{j}) \rangle = 0 
$$

\begin{figure}[!h]
\centering
\includegraphics[width=0.55\linewidth]{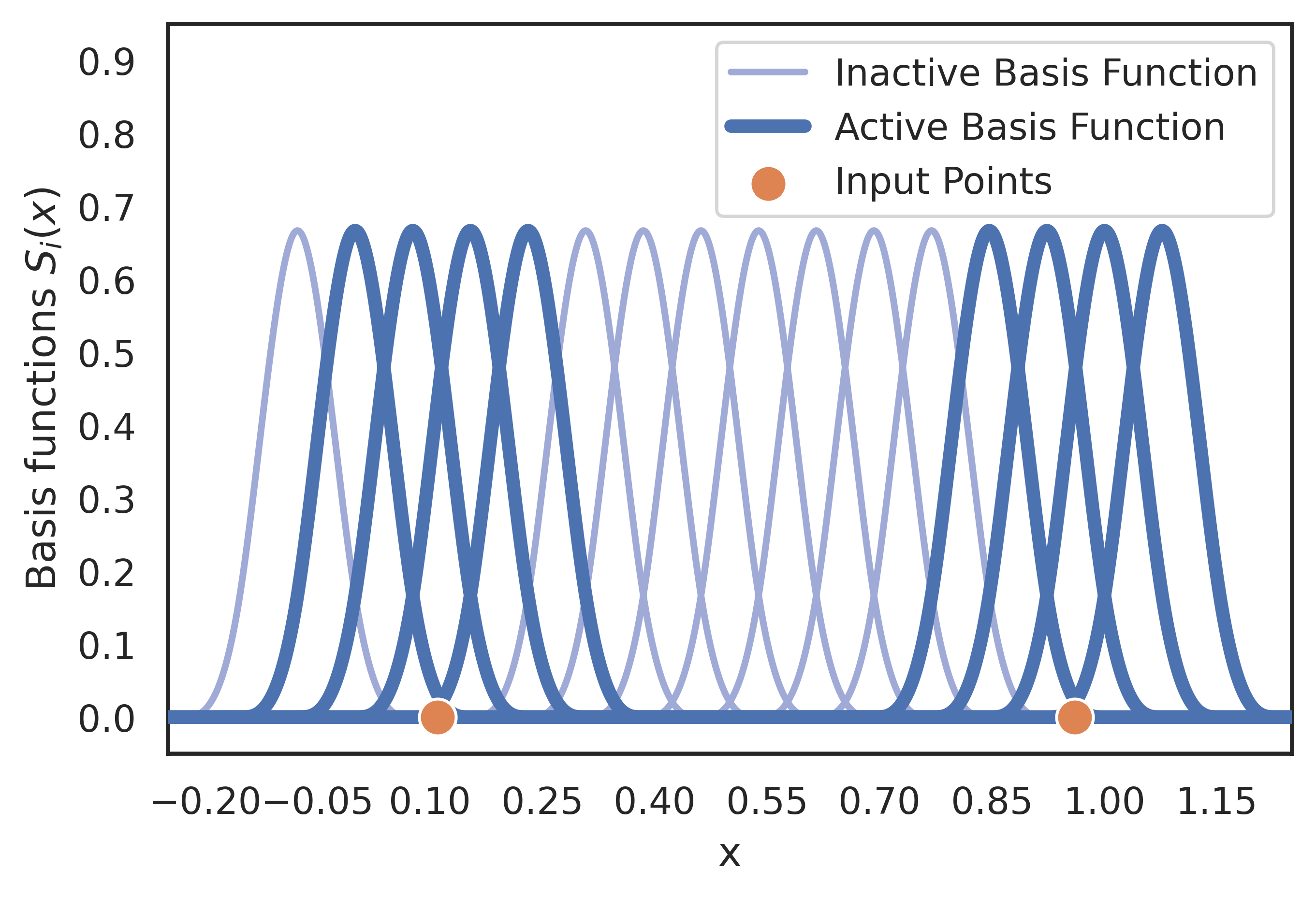}
\caption{Visual proof of distal orthogonality for single-variable $\rho$-density B-splines.}
\label{fig:fig_distal_orhogonality}
\end{figure}

This is because each basis function is zero everywhere, except on some small sub-interval. If this is true for all $j=1,...,n$, then the parameter gradients evaluated at $\vec{\mathbf{x}}$ and $\vec{\mathbf{y}}$ must be orthogonal. If this is true for all $j=1,...,n$, then it is true for the minimum. The converse is true by transitivity such that:

$$ 
|x_{j} - y_{j}|  > \delta \; \forall j=1,...,n 
\iff 
\min_{j=1, \dots , n }
\{ |x_{j} - y_{j}| \} > \delta
$$

Finally, 

$$
 \min_{j=1, \dots , n }
\{ |x_{j} - y_{j}| \} > \delta 
\implies  
\langle
\grad_{\vec{\mathbf{\theta}}} A(\vec{\mathbf{x}}) 
, 
\grad_{\vec{\mathbf{\theta}}} A(\vec{\mathbf{y}})
\rangle
= 0
$$

\end{proof}

\begin{remark}
Two points that sufficiently differ in each input variable have orthogonal parameter gradients. It is worth mentioning that the condition resembles a cross-like region in two variables, and planes that intersect in higher dimensions. Distal orthogonality means Atlas is reasonably robust to catastrophic forgetting.
\end{remark}

\end{document}